\documentclass{SCIS2025}

\usepackage{amsmath,graphicx,amssymb,amsthm,float,algorithm,algorithmic,mathtools,booktabs,multirow,color,url,float}
\usepackage{microtype,alphalph}

\newcommand{\tablestyle}[2]{\setlength{\tabcolsep}{#1}\renewcommand{\arraystretch}{#2}\centering}
\definecolor{mydarkblue}{rgb}{0,0.08,0.45}
\definecolor{mydarkgreen}{RGB}{0, 139, 69}
\hypersetup{
	colorlinks=true,
	urlcolor=magenta,
	citecolor=mydarkblue,
}

\begin{document}
\ArticleType{RESEARCH PAPER}
\Year{2025}
\Month{}
\Vol{}
\No{}
\DOI{}
\ArtNo{}
\ReceiveDate{}
\ReviseDate{}
\AcceptDate{}
\OnlineDate{}
\AuthorMark{}

\title{Task aware dreamer for task generalization in reinforcement learning}{Task aware dreamer for task generalization in reinforcement learning}

\author[1,\dag]{Chengyang Ying}{}
\author[1,\dag]{Xinning Zhou}{}
\author[1]{Zhongkai Hao}{}
\author[1, *]{\\ Hang Su}{{suhangss@mail.tsinghua.edu.cn}}
\author[1]{Songming Liu}{}
\author[1]{Dong Yan}{}
\author[1,*]{Jun Zhu}{{dcszj@mail.tsinghua.edu.cn}}

\contributions{These authors contributed equally to this work and are ordered by the dice.}

\address[1]{Department of Computer Science and Technology, \\ Beijing National Research Center for Information Science and Technology, \\ Tsinghua-Bosch Joint Center for Machine Learning, Institute for Artificial Intelligence, \\ Tsinghua University, Beijing 100084, China}

\abstract{
A long-standing goal of reinforcement learning is to achieve agents that can learn on various training tasks and generalize well on unseen tasks with similar transition dynamics but different reward functions. For instance, a household robot needs to use the same embodiment to complete various tasks characterized by distinct reward functions.
To develop agents capable of handling this challenge, we construct Reward-Informed World Models (RIWM) that learns task-specific features while utilizing similar transition dynamics of trajectories collected from different tasks.
The world model is optimized with a novel algorithm named Task Aware Dreamer (TAD), which employs a task context term to guide the world model to differentiate tasks and thus enhance task generalization.
We further provide theoretical analyses to show that TAD's key components are essential for handling task generalization.
To establish the necessity of the policy hypothesis set utilized in TAD, we introduce a novel metric named Task Distribution Relevance (TDR), which quantitatively measures the relevance of different tasks.
Tasks with high TAD often have vastly different optimal polices and thus are impossible to solve with traditional Markovian policy hypothesis sets, thereby necessitating the reward-informed policy used in TAD.
Extensive experiments in both image-based and state-based benchmarks show that TAD can significantly improve the performance of handling various training tasks, especially for those with high TDR, and exhibits strong zero-shot generalization ability to unseen tasks. We also conduct ablation studies that demonstrate TAD’s potential to handle transition variations in cross-embodiment scenarios.
}


\maketitle


\section{Introduction}

Although Deep Reinforcement Learning (DRL) agents have demonstrated significant advancements in various fields~\cite{1,2}, their achievements are usually limited to specific training tasks. 
This tendency toward specialization harms the real-world applicability of DRL agents, where robust generalization across diverse tasks is necessary. For example, a household robot needs to handle different tasks that share the same underlying transition dynamics (i.e., the same embodiment and physical environment) but are characterized by distinct reward functions.
Developing generalizable agents that can recognize and handle these subtle variations remains an active and important area of research in DRL.

Compared to the traditional single-task setting, we formalize the problem as the task-distribution setting (Sec.~\ref{sec_policy}), where each training/testing task is sampled from a task distribution $\mathcal{T}$.
To develop agents capable of handling this situation, we show that the world model~\cite{3,4,5} is a promising pathway to address the challenges posed by the variability of tasks. In this work, we first demonstrate that \textit{world models can enhance the sample efficiency when handling the task-distribution setting}, particularly by exploiting shared dynamic structures of trajectories sampled from different tasks (Theorem~\ref{thm-model}). While current world models are primarily tailored for the single-task setting (left of Fig.~\ref{pgm_tad_fig}), we extend the probabilistic graphical model to handle the task-distribution setting (right of Fig.~\ref{pgm_tad_fig}) by introducing the task random variable $\mathcal{M}$. 
Building upon this framework, we extend the traditional world~\cite{3,4,5} models to our Reward-Informed World Models (RIWM, Eq.~\eqref{eq_world_model}). RIWM incorporates historical reward signals to optimize latent states and integrates a task model to infer the task context of the trajectory.
We further propose a novel algorithm named \textbf{T}ask \textbf{A}ware \textbf{D}reamer (\textbf{TAD}) to optimize the aforementioned RIWM by maximizing the variational evidence lower bound.
The derived task context term encourages the world model to learn compact representations that effectively distinguish among tasks, thereby enhancing the agents’ ability to generalize across tasks. In practice, we implement this objective using two alternative approaches: a cross-entropy–based method (TAD-CE) and a supervised contrastive method (TAD-SC).

\begin{figure}[t]
\centering
\subfloat[Overview of this work\label{fig_overview}.]{
    \centering
    \includegraphics[height=4.2cm,width=6.4cm]{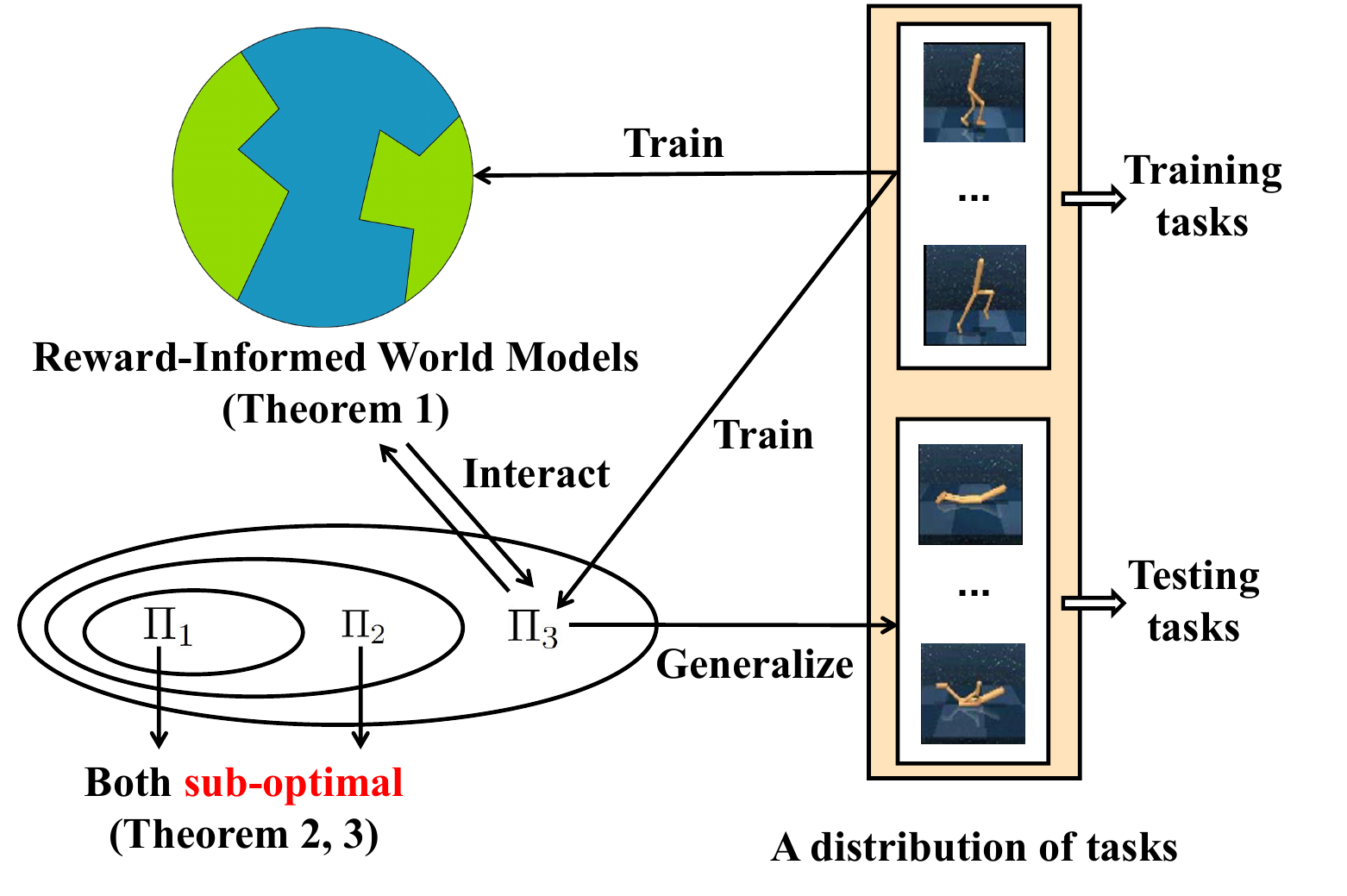}
}
\subfloat[Probabilistic graphical model designs\label{pgm_tad_fig}.]{
    \centering
    \includegraphics[height=4.2cm,width=7.8cm]{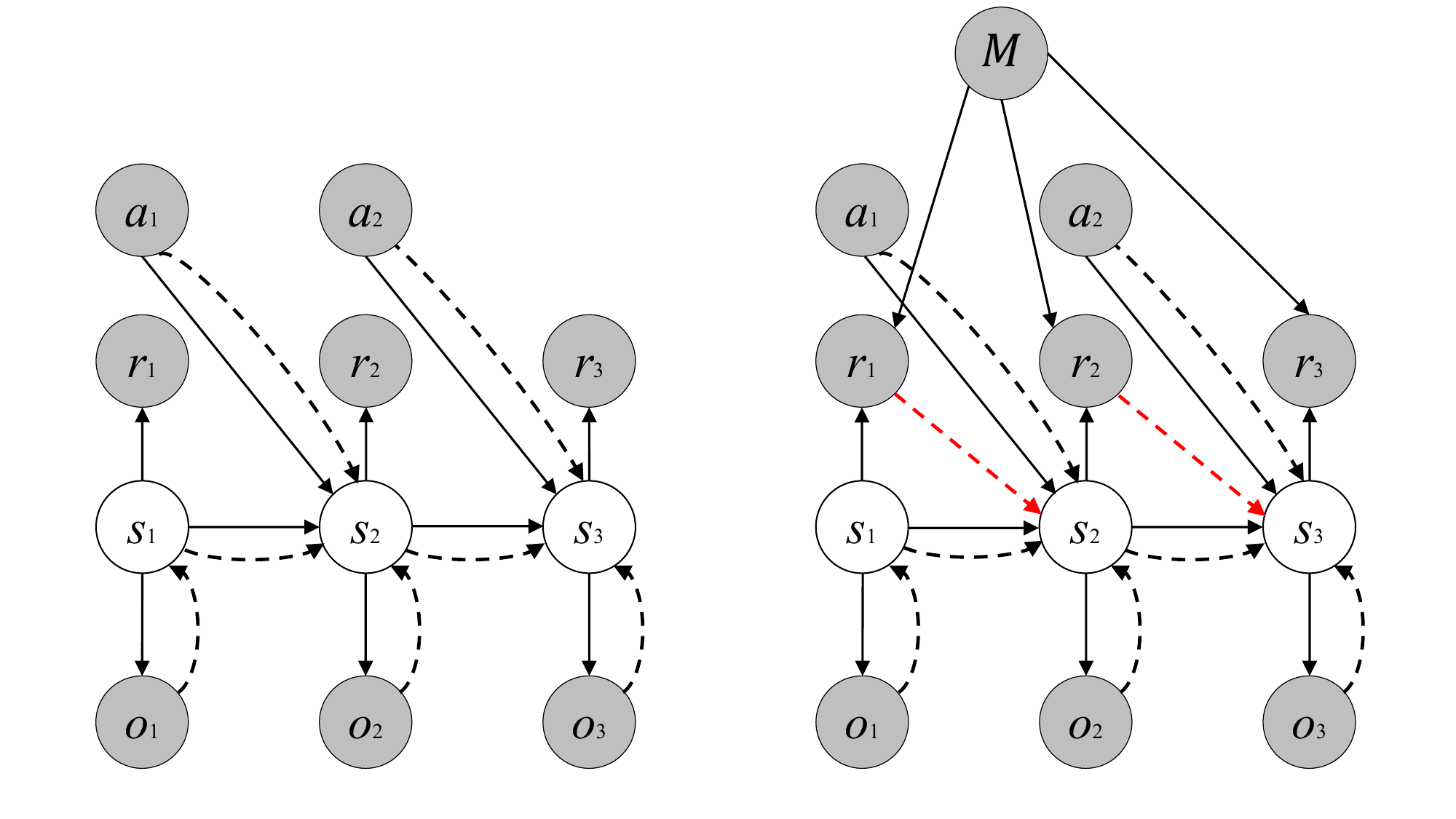}
}
\caption{(a) \textbf{An overview.} Given a task distribution, the agent is trained on the training tasks and is expected to generalize to the test tasks in a zero-shot manner.
Here $\Pi_1$ represents the set of Markovian policies, while $\Pi_2$ denotes all policies that include historical state and action information.
To improve generalization ability, we propose TAD, which utilizes $\Pi_3$ to encode all historical information for inferring the current task and introduces novel Reward-Informed World Models to capture task-specific latent features.
(b) Probabilistic graphical model designed for the single-task setting (\textbf{left}) and the task-distribution setting (\textbf{right}). The latter illustrates the design of Reward-Informed World Models, where solid and dashed lines denote the \emph{generative process} and the \emph{inference model}, respectively.
}
\end{figure}

To validate the effectiveness of the proposed method, we theoretically demonstrate that the key innovations of TAD are essential for addressing task generalization. Specifically, we first establish the necessity of adopting the reward-informed policy hypothesis $\Pi_3$ in the task-distribution setting. Although several meta RL methods~\cite{6,7} practically choose $\Pi_3$, the relationship between the task distribution and the expressiveness of the policy hypothesis set remains insufficiently understood.
In particular, we address the following question: \textit{Why are the commonly adopted hypothesis set $\Pi_1$ of Markovian policies and set $\Pi_2$ of policies that encode historical states and actions inadequate for the task-distribution setting?}

To answer this question, we propose a novel metric named Task Distribution Relevance (TDR), which encapsulates the relevance among tasks within the distribution through their optimal Q-functions. 
We then prove that both $\Pi_1$ and $\Pi_2$ are \textit{sub-optimal} in the task-distribution setting. 
This sub-optimality is characterized by TDR (Theorem~\ref{thm-2}). Specifically, for task distributions with high TDR, the performance of these two policy hypothesis sets can degrade significantly, a phenomenon we also empirically demonstrate in experiments (Sec .~\ref{sec_exp_tdr}).
This result provides a theoretical justification for the adoption of $\Pi_3$ in TAD as well as in prior meta-RL methods when addressing task-distribution settings.
Beyond the choice of policy hypothesis set, we further analyze the task context term in TAD and prove that it effectively reduces the gap between the policy return and the optimal return (Theorem~\ref{thm-4}).

We conduct extensive experiments to evaluate TAD’s ability to enhance task generalization, using commonly adopted benchmarks from the DeepMind Control Suite~\cite{8} and MuJoCo~\cite{9}. Both training and testing tasks are drawn from a task distribution $\mathcal{T}$, which requires the agents to simultaneously learn from multiple training tasks and generalize to potentially unseen test tasks.
The experimental results corroborate our theoretical analysis, showing that policies from $\Pi_1$ and $\Pi_2$ struggle to handle tasks from distributions with a high TDR.
Conversely, TAD not only effectively handles diverse training tasks but also exhibits better zero-shot generalization performance than the few-shot adaptation performance of the SOTA meta-RL method MAMBA~\cite{7}. Moreover, our ablation studies demonstrate the versatility of TAD, highlighting its potential for dynamics generalization and handling cross-embodiment task settings.

Overall, our contributions are summarized as follows:
\begin{itemize}
\item We provide theoretical insights showing that world models can exploit shared transition dynamics across tasks to improve sample efficiency for task generalization (Sec.~\ref{sec_wm}).
\item We extend world models to the task-distribution setting by proposing TAD, which introduces a task context term that encourages representations to capture task-relevant information (Sec.~\ref{sec_tad}).
\item We theoretically demonstrate the effectiveness of TAD’s components for task generalization and introduce a novel metric named TDR to quantify task distribution relevance (Sec.~\ref{theoretical}).
\item Extensive experiments show that TAD outperforms a range of SOTA baselines and exhibits strong generalization flexibility across both image-based and state-based benchmarks (Sec.~\ref{sec_expe}).
\end{itemize}

\section{Related work} 

\paragraph{Generalization in RL.} Current RL agents often struggle to generalize to new tasks~\cite{10}.
Prior work has explored a wide range of training strategies to address this challenge, including loss regularization~\cite{11,12}, successor representations~\cite{13,14}, network architecture design~\cite{15,16}, data augmentation~\cite{17,18}, leveraging trajectory similarities~\cite{19}, and contrastive objectives~\cite{20}.
Beyond these approaches, several studies have investigated the relationship between policy generalization and the distribution over environments. In particular, LEEP~\cite{21} demonstrates that generalizing to unseen environments introduces partial observability, which makes deterministic Markovian policies \textit{suboptimal}. Moreover, empirical evidence from prior works~\cite{22,21} suggests that stochastic or non-Markovian policies can enhance generalization ability.
Despite these insights, a theoretical gap remains regarding the expressiveness of different hypothesis sets and their relationship to the environment distribution, both of which are crucial for developing more generalizable agents.

\paragraph{Multi-task RL and Meta RL.} These two topics are closely related to generalization in RL. \textit{Multi-task RL}~\cite{23,24,25,26} primarily focuses on excelling across all training tasks, but often struggles to generalize to previously unseen tasks.
In contrast, \textit{Meta RL} aims to equip agents with the ability to rapidly adapt to new tasks with a few episodes, through approaches such as gradient-based~\cite{27} and context-based methods~\cite{28,29}.
In addition, several model-based approaches~\cite{30,7} leverage learned environment models to improve sample efficiency.
Although certain context-based methods, such as VariBAD~\cite{31}, demonstrate zero-shot generalization capabilities, it remains important to directly study the mechanisms of generalization in RL and to develop algorithms specifically tailored to this objective.

\paragraph{World models.} World models~\cite{3} aim to learn rich representations of the environment, offering potential advantages for generalization by capturing task-invariant features.
Classical approaches typically adopt the Recurrent State Space Model (RSSM)~\cite{4} for planning~\cite{4} and policy learning~\cite{5,32,33}.
Subsequent work has proposed reconstruction-free world models~\cite{34}, temporal predictive coding~\cite{35}, cooperative reconstruction~\cite{36}, and Denoised MDPs~\cite{37} to more effectively encode task-relevant information.
World models have also been leveraged to extract environment-invariant features through videos~\cite{38} or unsupervised exploration~\cite{39,40}, followed by fine-tuning on downstream tasks.
Despite these advances, most existing world models are designed for single-task settings and require task-specific fine-tuning to handle multiple tasks, which limits their ability to achieve zero-shot generalization to unseen tasks.

\section{Preliminary}
\label{sec_policy}

We consider a setting with a task distribution $\mathcal{T}$ over Partially Observable Markov Decision Processes (POMDPs), where different tasks share identical dynamics but differ in their reward functions.
Formally, each POMDP $\mathcal{M} \sim \mathcal{T}$ is defined as $\mathcal{M}=(\mathcal{S},\mathcal{A}, \mathcal{P}, \mathcal{R}_{\mathcal{M}},\Omega, \mathcal{O})$.
Here, $\mathcal{S}$ and $\mathcal{A}$ denote the state and action spaces, respectively. For $\forall (s,a)\in\mathcal{S}\times \mathcal{A}$, $\mathcal{P}(\cdot|s,a)$ specifies the Markovian state transition dynamics, while $\mathcal{R}_{\mathcal{M}}(s,a)$ denotes the task-specific reward function.
The underlying state $s$ is not directly observable by the agent. Instead, the agent receives observations from the observation space $\Omega$ according to the observation function $\mathcal{O}(\cdot \mid s)$.

Following prior work on meta-RL and generalization, we consider policies that encode all historical information (we prove the necessity of this assumption in Sec.~\ref{theoretical}).
Formally, at each time step $t$, the agent with the policy $\pi$ will utilize the whole historical trajectory $(o_0,a_0,r_0,o_1,...,o_t)$ to sample action $a_{t}$, arrive at the next state $s_{t+1}\sim\mathcal{P}(\cdot|s_t,a_{t})$, and receive current reward $r_{t}=\mathcal{R}_{\mathcal{M}}(s_t,a_t)$. The performance of policy $\pi$ on task $\mathcal{M}$ is measured by the expected discounted return $J_{\mathcal{M}}(\pi) = \mathbb{E}_{\tau\sim\pi}\left[R(\tau)\triangleq\sum_{t=0}^{\infty}\gamma^t r_t\right]$. Our objective is to learn a policy that maximizes the expected return over the task distribution $\mathcal{T}$: $\max_{\pi} \mathbb{E}_{\mathcal{M}\sim\mathcal{T}} \left[J_{\mathcal{M}}(\pi)\right]$.

In practice, given the task distribution $\mathcal{T}$, we sample $M$ training tasks $\{\mathcal{M}_m\}_{m=1}^M$ to optimize the agent by maximizing the empirical objective $\frac{1}{M}\sum_{m=1}^M J_{\mathcal{M}_{m}}(\pi)$. At test time, we sample $N$ previously unseen tasks $\{\mathcal{M}_{M+n}\}_{n=1}^{N}$ to evaluate the agent’s generalization performance, measured by $\frac{1}{N}\sum_{n=1}^N J_{\mathcal{M}_{M+n}}(\pi)$.

\section{Task aware dreamer}
\label{sec_tdr}

In this section, we begin by establishing that world models are beneficial for task generalization. We then introduce Reward-Informed World Models (RIWM) and propose Task-Aware Dreamer (TAD) as a novel approach to promote generalization. Finally, we provide theoretical analyses demonstrating that TAD can effectively enhance performance across various tasks.

\subsection{Reward-informed world models}
\label{sec_wm}

Our first observation is that world models are effective in narrowing the hypothesis space of the optimal Q-function when handling task generalization:
\begin{theorem}[Proof in Appendix A.1]
\label{thm-model}
Set $\mathcal{Q}$ denotes the space of observation-action Q-functions. Given $M$ tasks $\{\mathcal{M}_m\}_{m=1}^M$ and corresponding dataset $\mathcal{D}_m=\{(o_t^m,a_t^m, r_t^m, o_{t+1}^m)\}$, we use $\mathcal{H} = \mathcal{Q}^M$ to represent the product space of $M$ Q-function spaces, i.e., $\forall \{q_m\}_{m=1}^M\in\mathcal{H}$, $q_m:\mathcal{S}\times\mathcal{A}\rightarrow\mathbb{R}$ belongs to $\mathcal{Q}$. Considering the following three hypothesis sets $\mathcal{H}_1, \mathcal{H}_2, \mathcal{H}_3 \subseteq \mathcal{H}$:
\begin{equation*}
\begin{split}
    \mathcal{H}_1 =& \{ (q_m)_{m=1}^M  | q_m (o_t^m, a_t^m) = r_t^m + \gamma \max_{a'} q_m (o_{t+1}^m, a') \} \\
    \mathcal{H}_2 =& \{ (q_m)_{m=1}^M | \exists (p_m)_{m=1}^M: p_m(o_t^m, a_t^m) = o_{t+1}^m,\\
    & \exists (r_m)_{m=1}^M: r_m(o_t^m, a_t^m) = r_{t}^m,
    q_m (o, a) = r_m(o,a) + \gamma \max_{a'} q_m (p_m(o, a), a'), \forall o,a \} \\
    \mathcal{H}_3 =& \{ (q_m)_{m=1}^M  |\exists p:p(o_t^m, a_t^m) = o_{t+1}^m,\\
    & \exists (r_m)_{m=1}^M: r_m(o_t^m, a_t^m) = r_{t}^m, 
    q_m (o, a) = r_m(o, a) + \gamma \max_{a'} q_m (p(o, a), a'), \forall o, a \}.
\end{split}
\end{equation*}
The hypothesis sets satisfy $\mathcal{H}_3\subseteq \mathcal{H}_2\subseteq \mathcal{H}_1$.
\end{theorem}

Here, $\mathcal{H}_1$ contains Q-functions that satisfy the optimal Bellman equation with data from $\mathcal{D}_m$. In contrast, $\mathcal{H}_2$ and $\mathcal{H}_3$ contain Q-functions that satisfy the optimal Bellman equation using data generated by world models trained on a single dataset and on all datasets, respectively.  
Consequently, $\mathcal{H}_2\subseteq\mathcal{H}_1$ utilizes the generalization ability of world models, extending prior results in the fixed-task setting~\cite{41}. Moreover, $\mathcal{H}_3\subseteq\mathcal{H}_2$ suggests that shared dynamic structures across tasks benefit learning world models.  
Overall, Theorem~\ref{thm-model} indicates that training world models narrows the hypothesis space of plausible Q-functions, thereby facilitating agent learning and improving generalization.

As existing world models~\cite{4,5} are primarily designed for the single-task setting, we propose Reward-Informed World Models (RIWM) for the task-distribution setting.
We begin by analyzing the probabilistic graphical model of this setting, illustrated in Fig.~\ref{pgm_tad_fig}, where the rewards depend not only on previous states and actions but also on the current task. Accordingly, the joint distribution can be expressed as:

\begin{equation}
\begin{aligned}
    &p(s_{1:T},o_{1:T},r_{1:T},a_{1:T-1},\mathcal{M})
    = p(\mathcal{M})\prod_{t=1} p(s_{t+1} | s_t, a_t) p(o_t|s_t) p(r_t | s_t, \mathcal{M}),
\end{aligned}
\end{equation}
and we choose to adopt the reward-informed inference model (the necessity of reward-informed policy in task generalization is discussed in Sec.~\ref{theoretical}) to approximate state posteriors as:
\begin{equation}
\begin{aligned}
    q(s_{1:T}|o_{1:T},a_{1:T},r_{1:T}) = \prod_{t=1}^T q(s_t|s_{t-1}, a_{t-1}, r_{t-1}, o_t).
\end{aligned}
\end{equation}

Based on this inference model, we can construct the variational lower bound of the action-conditioned probability as:
\begin{equation}
\begin{split}
\label{eq_elbo}
    \ln p(o_{1:T},r_{1:T}, \mathcal{M} |a_{1:T})
    \ge &\sum_{t=1}^T \mathbb{E}_{q(s_{t}|o_{\leq t},a_{< t},r_{< t})} \left[ \ln p(o_t, r_t|\mathcal{M}, s_t) - \mathrm{KL}\left(q(s_t|o_{\leq t}, r_{< t}, a_{< t})\| p(s_t|s_{t-1}, a_{t-1})\right)\right] \\
    &+ \mathbb{E}_{q(s_{1:T}|o_{1:T},a_{1:T},r_{1:T})}\left[\ln p(\mathcal{M} |s_{1:T})
    \right].
\end{split}
\end{equation}

This result is a general form of the single-task setting~\cite{4}, with the full derivation in Appendix A.2. The first two terms in  Eq.~\eqref{eq_elbo} correspond to reconstructing observations and predicting rewards, which are similar to the single-task setting. The final novel term is dedicated to predicting the current task from historical information, which enhances generalization by encouraging the model to infer the current task context.
Building on RSSM via Eq.~\eqref{eq_elbo}, our RIWM consists of:

\begin{equation}
\begin{aligned}
\label{eq_world_model}
    &\text{Deterministic state model:}& & h_t = f(h_{t-1}, s_{t-1}, a_{t-1}, r_{t-1}), & \\
    &\text{Transition model:} & &p_{\theta}(s_t | h_t),& \\
    &\text{Observation model:}& & p_{\theta}(o_t | h_t, s_t), & \\
    &\text{Reward model:} & &p_{\theta}(r_t | h_t, s_t),&\\
    &\text{Task model:} & & p_{\theta}(\mathcal{M}|h_t,s_t).&
\end{aligned}
\end{equation}

Here, the hidden state $h_t$ encodes historical states, actions, and rewards using a gated recurrent unit (GRU)~\cite{42}, i.e., $f(\cdot) = \text{GRU}(\cdot)$. The transition, observation, reward, and task models then predict the next state, observation, reward, and task context, respectively.

\subsection{Optimization}
\label{sec_tad}
Based on the above analyses, we now describe the training of TAD in detail. Following prior world model approaches~\cite{5}, TAD adopts an alternating training scheme between the RIWM and the policy.

To ensure balanced learning across tasks, TAD maintains $M$ replay buffers $\{\mathcal{D}_m\}_{m=1}^M$ to store trajectories collected from the corresponding task $\{\mathcal{M}_m\}_{m=1}^M$. During training, TAD samples data from each replay buffer and updates the RIWM by optimizing the objective in Eq.~\eqref{eq_elbo} as:

\begin{equation}
\begin{split}
\label{eq_loss}
    &L_{\text{TAD}} = \sum_{i=1}^M \mathbb{E}_q \left[  \sum_{t=1}^T \ln p_{\theta} (o_t^i|h_t^i, s_t^i) + \sum_{t=1}^T \ln p_{\theta} (r_t^i | h_t^i, s_t^i)\right.
    - \left.\sum_{t=1}^T D_{\text{KL}}(q(s_t^i|h_t^i, o_t^i) \| p_{\theta}(s_t^i|h_t^i))  + 
    L_{\text{task}}\right].
\end{split}
\end{equation}

In Eq.~\eqref{eq_loss}, the first three terms are similar to Dreamer for reconstructing observations, predicting rewards, and inferring states.
In addition, TAD introduces a novel task context term $L_{\text{task}}$ to distinguish tasks and learn task-aware embeddings. Specifically, we propose two alternative formulations for this term: cross-entropy and self-contrastive, which are described in detail below. A theoretical justification for their effectiveness is provided in Sec.~\ref{theoretical}.

\paragraph{Cross-entropy.} The final term in Eq.~\eqref{eq_elbo} emphasizes the need to maximize the log-probability of the task context, enabling the model to distinguish tasks based on historical information. Following prior work~\cite{23}, TAD-CE directly maximizes this log-probability over different tasks represented by one-hot vectors. The task loss $L_{\text{task}}$ is thus defined as:

\begin{equation}
\begin{split}
    L_{\text{task}} = \sum_{t=1}^T \ln p_{\theta} (m^i | h_t^i, s_t^i).
\end{split}
\end{equation}

\paragraph{Supervised-contrastive.} In addition to directly maximizing the log-probability, we further propose TAD-SC, inspired by supervised contrastive learning~\cite{43}. Specifically, TAD-SC encourages task embeddings from the same task to be close while pushing embeddings from different tasks apart, enabling RIWM to distinguish tasks more effectively. Formally, we assume that the task model maps all sampled data as $\{m_j\}_{j=1}^{M\times T}$ and define $L_{\text{task}}$ as:

\begin{equation}
\begin{split}
    L_{\text{task}} = \sum_{j=1}^{M\times T} \sum_{a\in A(j)} \ln \frac{\exp(m_j\cdot m_a / \tau)}{\sum_{b\neq j} \exp(m_j\cdot m_b / \tau)},
\end{split}
\end{equation}
where $A(j)$ is the set of data indices that are collected from task $m_j$, and $\tau$ is the temperature parameter, which is set as 0.1 following previous works~\cite{43}.

For training the actor-critic networks, we extend the actor-critic learning in Dreamer to the task distribution. 
Specifically, we first sample a set of states from the replay buffers and use our RIWM to imagine trajectories starting from these states (illustrated in Fig.~\ref{fig_imagine}). These imagined trajectories can capture task-specific features and are thus beneficial for the agent to gain better generalization.
Once the imagined trajectories are obtained, the actor and critic are optimized by maximizing the $\lambda$-return~\cite{44} and regressing the TD targets~\cite{45}, respectively.

\begin{algorithm}[t]
    \caption{Task Aware Dreamer (TAD)} 
    \label{algo_tad}
    \begin{algorithmic}[1]
        \REQUIRE $M$ training tasks $\{\mathcal{M}_m\}_{m=1}^M$, $M$ replay buffers $\{\mathcal{D}_m\}_{m=1}^M$, $N$ test tasks $\{\mathcal{M}_{M+n}\}_{n=1}^N$, initialize parameters of world models, the policy, and the critic.
        \FOR{$\text{iteration step} = 1,2,...$}
        \FOR{$\text{update step} = 1,2,...$}
        \STATE Sample $o$-$a$-$r$ pairs $\{(o_t^i,a_t^i,r_t^i)_{t=1}^{T}\}$ form each replay buffer $\mathcal{D}_i, i=1,2,...,M$
        \STATE Calculate the deterministic state $h$ and further calculate model states $s$.
        \STATE Update the world models via optimizing Eq.~\eqref{eq_loss}.
        \STATE Collect imagined trajectories from each $s$ via the policy and the world models and use these imagined trajectories to update the policy and the critic.
        \ENDFOR
        \STATE Collect trajectories from $\mathcal{M}_m (m=1,2,...,M)$ and store them into the replay buffer $\mathcal{D}_m$.
        \ENDFOR
    \STATE Evaluate the agent in testing environments $\{\mathcal{M}_{M+n} \}$.
    \end{algorithmic}  
\end{algorithm}

\subsection{Theoretical analyses}
\label{theoretical}

Below, we provide theoretical analyses to show that TAD's components are simple but effective for task generalization. 

\paragraph{Are policies that utilize all historical information in TAD necessary for task generalization?}
Below, we introduce $3$ types of widely used policy hypothesis sets and show that $\Pi_3$, used in TAD, is necessary for handling task generalization.
\begin{enumerate}
\item Markovian policy set $\Pi_1$~\cite{45,46}, i.e., $\Pi_1 = \{\pi | \pi:\mathcal{S}\rightarrow \Delta(\mathcal{A})\}$, where $\Delta(\mathcal{A})$ represents a distribution over $\mathcal{A}$, which is widely used and optimal for the single-task setting;
\item $\mathcal{S}$-$\mathcal{A}$ memorized policy set $\Pi_2$~\cite{5,32,22}, i.e., $\Pi_2 = \{\pi | \pi:\mathcal{H}\rightarrow \Delta(\mathcal{A})\}$, where $\mathcal{H} = \cup_{t=1}^{\infty}\mathcal{H}_t, \mathcal{H}_t = (\mathcal{S}\times \mathcal{A})^{t-1} \times \mathcal{S}$;
\item $\mathcal{S}$-$\mathcal{A}$-$\mathcal{R}$ memorized policy set $\Pi_3$~\cite{31,7}, 
i.e., $\Pi_3 = \{\pi | \pi:\mathcal{L}\rightarrow \Delta(\mathcal{A})\}$, where $\mathcal{L} = \cup_{t=1}^{\infty}\mathcal{L}_t, \mathcal{L}_t = (\mathcal{S}\times \mathcal{A}\times \mathbb{R})^{t-1} \times \mathcal{S}$.
\end{enumerate}

As illustrated in Fig.~\ref{fig_overview}, these hypothesis sets naturally satisfy $\Pi_1\subseteq \Pi_2\subseteq \Pi_3$. 
We then analyze their expressive ability under task distribution $\mathcal{T}$.
Let $J_{\mathcal{T}}^* \triangleq \mathbb{E}_{\mathcal{M}\sim\mathcal{T}} \left[\max_{\pi}J_{\mathcal{M}}(\pi)\right]$ denote the optimal return under $\mathcal{T}$, and $J_{\mathcal{T}}^i \triangleq \max_{\pi\in\Pi_i}\left[ \mathbb{E}_{\mathcal{M}\sim\mathcal{T}} J_{\mathcal{M}}(\pi)\right]$ the optimal return achievable by policies in $\Pi_i(i=1,2,3)$ under $\mathcal{T}$.
Our first result shows that, although $\Pi_1\subseteq \Pi_2$, they own the same expressive ability, i.e., $J_{\mathcal{T}}^1=J_{\mathcal{T}}^2$, and are both \emph{sub-optimal}:

\begin{theorem}[Sub-Optimality of $\Pi_1,\Pi_2$. Proof in Appendix A.3]
\label{thm-1}
We set $\bar{\mathcal{M}}=(\mathcal{S},\mathcal{A}, \mathcal{P}, \bar{\mathcal{R}},\gamma)$, here $\bar{\mathcal{R}} = \mathbb{E}_{\mathcal{M}\sim\mathcal{T}}[\mathcal{R}_{\mathcal{M}}]$.
For $\forall \pi \in \Pi_2$, we have $\mathbb{E}_{\mathcal{M}\sim\mathcal{T}} [J_{\mathcal{M}}(\pi)] = J_{\bar{\mathcal{M}}}(\pi)$ and further $J_{\mathcal{T}}^1=J_{\mathcal{T}}^2\leq J_{\mathcal{T}}^*$.
\end{theorem}

Theorem~\ref{thm-1} reveals that the cumulative returns of policies in $\Pi_1$ and $\Pi_2$ are equivalent as their returns in the ``average'' MDP $\bar{\mathcal{M}}$, where the reward function is the average of reward functions in different tasks. Moreover, since policies in $\Pi_1$ and $\Pi_2$ only select actions based on current state or historical state-action pairs, they are unable to distinguish different tasks and are thus both \emph{sub-optimal}.

To quantitatively analyze the relationship between $\mathcal{T}$ and the gap between $J_{\mathcal{T}}^1, J_{\mathcal{T}}^2$ and $J_{\mathcal{T}}^*$, we propose a novel metric named Task Distribution Relevance (TDR), which characterize the distribution $\mathcal{T}$ as:
\begin{definition}[Task Distribution Relevance]
\label{df_tdr}
For any task distribution $\mathcal{T}$ and state $s$, the Task Distribution Relevance of $\mathcal{T}$ and $s$ is defined as:
\begin{equation}
\begin{split}
    D_{\text{TDR}}(\mathcal{T}, s) &= \mathbb{E}_{\mathcal{M}\sim\mathcal{T}} [\max_a Q_{\mathcal{M}}^*(s,a)] 
    - \max_{a}\mathbb{E}_{\mathcal{M}\sim\mathcal{T}} \left[ Q_{\mathcal{M}}^*(s,a) \right].
\end{split}
\end{equation}
\end{definition}
Intuitively, TDR quantifies the relevance of tasks in $\mathcal{T}$ via optimal $Q$ functions, which determine the choice of optimal actions in corresponding tasks. Based on TDR, we can derive a bound on the performance gap:

\begin{theorem}[Regret of Markovian Policies]
\label{thm-2}
Assume $\pi_{\mathcal{M}}^* = \mathop{\arg\max}_{\pi} J_{\mathcal{M}}(\pi)$, for $\forall \pi\in\Pi_1$, we have
\begin{equation}
\begin{split}
    J_{\mathcal{T}}^* - \mathbb{E}_{\mathcal{M}\sim\mathcal{T}}\left[J_{\mathcal{M}}(\pi)\right]
    \ge &\frac{1}{1-\gamma}\mathbb{E}_{s\sim d_{\mathcal{M},\pi}} [D_{\text{TDR}}(\mathcal{T}, s)].
\end{split}
\end{equation}

Thus $J_{\mathcal{T}}^* - J_{\mathcal{T}}^2 = J_{\mathcal{T}}^* - J_{\mathcal{T}}^1\ge \frac{1}{1-\gamma}\mathbb{E}_{s\sim d_{\mathcal{M},\pi^*}} \left[D_{\text{TDR}}(\mathcal{T}, s)\right]$, here $\pi^* = \arg\max_{\pi\in\Pi_1}J_{\bar{\mathcal{M}}}(\pi)$.

\end{theorem}

Theorem~\ref{thm-2} demonstrates that the gap between $J_{\mathcal{T}}^1,J_{\mathcal{T}}^2$ and $J_{\mathcal{T}}^*$ is closely related to the TDR of $\mathcal{T}$. 
For task distributions $\mathcal{T}$ with high TDR, i.e., the optimal $Q$ values vary greatly across tasks, the performance of $\Pi_1$ and $\Pi_2$ can be extremely poor since they cannot differentiate between tasks and their expressive abilities are significantly limited. This conclusion is further confirmed empirically in the experimental section. Moreover, we show that $J_{\mathcal{T}}^1, J_{\mathcal{T}}^2$ can be arbitrarily small when $J_{\mathcal{T}}^3$ can be arbitrarily close to $J_{\mathcal{T}}^*$ (see Appendix A.5 for details), demonstrating that $\Pi_3$ possesses stronger expressive ability.
Consequently, it is necessary to utilize policies in $\Pi_3$ to distinguish different tasks and thus improve generalization across the task distribution.

\paragraph{Is optimizing $p_{\theta}(\mathcal{M}|h_t,s_t)$ helpful for task generalization?} Now we will show that, although $p_{\theta}(\mathcal{M}|h_t,s_t)$ is a simple term, optimizing it can be effective for obtaining a generalizable agent in the task distribution setting.

As the input space of $\Pi_3$ is $\mathcal{L}$, i.e., all historical trajectories, our main result analyzes the relationship between a policy $\pi\in\Pi_3$ and the task posterior $p(\mathcal{M}|l), l\in\mathcal{L}$. 
As $h_t$ encodes all historical information, $p_{\theta}(\mathcal{M}|h_t,s_t)$ can equivalently be expressed as $p(\mathcal{M}|l), l\in\mathcal{L}$.

\begin{theorem}[Informal; detailed analyses and proof in Appendix A.6]
\label{thm-4}
For $\forall \pi\in\Pi_3$, we have
\begin{equation}
\begin{split}
\label{eq_pi3}
    J_{\mathcal{T}}^* - \mathbb{E}_{\mathcal{M}\sim\mathcal{T}}\left[J_{\mathcal{M}}(\pi)\right]
    =& \frac{1}{1-\gamma} \int_{\mathcal{L}} p(l) \left[\int  p(\mathcal{M}|l) \max_a Q_{\mathcal{M}}^*(l,a) d\mathcal{M}\right.
    -\left. \int_{a,\mathcal{M}} \pi(a|l)  p(\mathcal{M}|l) Q_{\mathcal{M}}^*(l,a)dad\mathcal{M} \right] dl,
\end{split}
\end{equation}
where $p(l)$ is a distribution of $\mathcal{L}$ related to $\mathcal{T},\pi$ and $p(\mathcal{M}|l)$ is the task posterior related to $\pi$.
\end{theorem}

Consequently, maximizing $p(\mathcal{M}|l)$, i.e., making the distribution of $p(\mathcal{M}|l)$ closer to a certain Dirac distribution, can significantly reduce the right part of Eq.~\eqref{eq_pi3}, and thus is effective for improving the generalization ability of $\pi$. Additional details and discussion are provided in Appendix A.6.

\section{Experiments}
\label{sec_expe}

We now present empirical results to answer the following questions:

\begin{itemize}
    \item Can we verify the analyses about TDR, i.e., the expressive abilities of $\Pi_1,\Pi_2$ are severely restricted in task distributions with high TDR? (Sec.~\ref{sec_exp_tdr})
    \item How well does TAD generalize across tasks when handling image-based and state-based observations? (Sec.~\ref{sec_exp_gene})
    \item Can TAD be extended to more general settings, such as dynamics generalization? (Sec.~\ref{sec_exp_abla})
\end{itemize}

\subsection{Experimental setup}
\label{sec_expe_setup}

\paragraph{Image-based control.} To verify our analyses of TDR in Sec.~\ref{theoretical}, we consider several task combinations in the DeepMind Control suite (DMC)~\cite{8}: (1) \textbf{Cartpole-balance\&balance$\_$sparse}, which shares the same optimal actions with TDR$=0$; (2) \textbf{Walker-stand\&walk\&prostrat\&flip}, and \textbf{Cheetah-run\&run$\_$back\&flip\&flip$\_$flip$\_$back}, which are widely used in multi-task unsupervised RL~\cite{39,47} and include task combinations with non-zero TDR. 
For example, Cheetah-run and Cheetah-flip hope a two-legged robot to move forward by running and flip around the torso, respectively (see Fig.~\ref{fig_imagine}), yielding distinct optimal Q-functions and thus leading to huge TDR. Additional details about these task combinations are provided in Appendix C.1.

To evaluate generalization with image-based observations, we extend classic tasks in DMC and design three task distributions: (1) \textbf{Cheetah$\_$speed($\alpha,\ \beta$)}, which extends Cheetah-run and hopes the agent to run within the target speed interval $(\alpha-\beta,\ \alpha+\beta)$; (2) \textbf{Pendulum$\_$angle($\alpha,\ \beta$)}, extending Pendulum-swingup to keep pendulum's pole within the target angle interval $(\arccos\alpha,\ \arccos\beta)$; and (3) \textbf{Walker$\_$speed($\alpha,\ \beta$)}, which is based on Walker-run and requires the planar walker to run within the target speed interval $(\alpha-\beta,\ \alpha+\beta)$.
For each task distribution, we sample $4$ training tasks and $2$ additional test tasks. More details are provided in Appendix C.2.

\paragraph{State-based control.} To demonstrate the scalability of TAD, we also consider several state-based continuous robotic control task distributions simulated via MuJoCo~\cite{48}. Following previous work~\cite{27,29}, 
we select the following task distributions: (1) \textbf{Half-Cheetah-Fwd-Back}, consisting of two opposite tasks; (2) \textbf{Half-Cheetah-Vel} and \textbf{Humanoid-Direc-2D}, each comprising 100 training tasks and 30 test tasks. More details of these task distributions are provided in Appendix D.

\paragraph{Baselines.} In DMC, we choose two model-free methods using Markovian policies $\Pi_1$: \textbf{CURL}~\cite{49} and \textbf{SAC+AE}~\cite{46}. In addition, we include two classic world models, \textbf{PlaNet}~\cite{4} and \textbf{Dreamer}~\cite{5}, and two advanced contrastive learning-based methods, TACO~\cite{50} and TD-MPC~\cite{51}, which utilize historical state-actions in policies and thus belong to $\Pi_2$. Moreover, we evaluate a SOTA model-based meta RL method belonging to $\Pi_3$: \textbf{MAMBA}~\cite{7}.
For state-based control, besides \textbf{PlaNet} and \textbf{Dreamer}, we include meta-RL baselines such as \textbf{MAML}~\cite{27}, \textbf{RL2}~\cite{28}, and \textbf{VariBAD}~\cite{31}, including zero-shot and few-shot evaluation as references.

\paragraph{Metrics.} For the task combinations, we evaluate the average return of all tasks to verify TDR. 
For the task generalization setting, agents are trained on the training tasks and evaluated on the test tasks to assess their generalization performance.
All experiments are repeated with 5 different random seeds, and results are reported as the mean $\pm$ std to mitigate the effects of randomness following previous works~\cite{5,7}.

\begin{table*}[t]
\centering
\tablestyle{3.5pt}{1.1}
\footnotesize
\begin{tabular}{cccccccccccccccc}
\toprule

\multirow{2}*{Algorithm} & \multirow{2}*{Hypothesis} & Cartpole-balance & Walker-stand\&walk  & Cheetah-run\&run$\_$back 
\\ 
& & \&balance$\_$sparse & \&prostrate\&flip & \&flip\&flip$\_$back 
\\
\midrule
CURL
& $\Pi_1$
& \textbf{994.5 $\pm$ 3.6}
& 254.1 $\pm$ 9.2
& 229.7 $\pm$ 10.9 
\\
CURL (w/ r)
& $\Pi_3$
& \textbf{987.3 $\pm$ 12.9}
& 265.0 $\pm$ 4.7
& 236.1 $\pm$ 5.8

\\
SAC+AE
& $\Pi_1$
& \textbf{992.5 $\pm$ 2.6}
& 256.9 $\pm$ 5.9
& 225.8 $\pm$ 10.1 
\\
SAC+AE (w/ r)
& $\Pi_3$
& \textbf{988.5 $\pm$ 7.0}
& 257.2 $\pm$ 8.5
& 239.9 $\pm$ 12.4
\\
PlaNet
& $\Pi_2$
& 309.5 $\pm$ 59.9
& 606.7 $\pm$ 152.9
& 244.8 $\pm$ 17.8 
\\
Dreamer
& $\Pi_2$
& \textbf{974.2 $\pm$ 5.8}
& 722.2 $\pm$ 12.6
& 241.1 $\pm$ 19.5 
\\
TD-MPC
& $\Pi_2$
& 614.0 $\pm$ 8.6
& 609.5 $\pm$ 49.6
& 467.4 $\pm$ 52.6
\\
TACO
& $\Pi_2$
& 740.0 $\pm$ 69.4
& 540.2 $\pm$ 139.2
& 200.6 $\pm$ 49.8
\\
MAMBA
& $\Pi_3$
& \textbf{994.7 $\pm$ 3.2}
& 436.8 $\pm$ 116.1
& 375.6 $\pm$ 44.0 
\\
TAD-CE (Ours)
& $\Pi_3$
& \textbf{998.9 $\pm$ 0.4}
& \textbf{778.9 $\pm$ 63.1}
& 549.6 $\pm$ 28.6 
\\
TAD-SC (Ours)
& $\Pi_3$
& \textbf{982.6 $\pm$ 2.0}
& \textbf{807.8 $\pm$ 85.2}
& \textbf{588.8 $\pm$ 20.2} 
\\
\bottomrule
\end{tabular}
\caption{Performance (mean $\pm$ std) in DMC.
Numbers greater than 95$\%$ of the best performance are \textbf{bold}.}
\label{table_multi_task}
\end{table*}

\begin{table*}[t]
\centering
\tablestyle{3.5pt}{1.1}
\footnotesize
\begin{tabular}{ccccccccc}
\toprule
\multirow{2}*{Algorithms}& \multirow{2}*{Hypothesis} & \multicolumn{2}{c}{Cheetah$\_$speed} & \multicolumn{2}{c}{Pendulum$\_$angle} & \multicolumn{2}{c}{Walker$\_$speed}\\
& & Train & Test & Train & Test & Train & Test \\
\midrule
CURL
& $\Pi_1$
& 211.7 $\pm$ 13.7
& 57.4 $\pm$ 26.6
& 140.2 $\pm$ 1.7
& 46.1 $\pm$ 29.8
& 127.0 $\pm$ 33.7
& 77.5 $\pm$ 11.5\\
SAC+AE
& $\Pi_1$
& 182.2 $\pm$ 7.6
& 115.2 $\pm$ 10.1
& 130.6 $\pm$ 12.2
& 89.0 $\pm$ 25.7
& 136.8 $\pm$ 34.4
& 27.5 $\pm$ 10.5\\
PlaNet
& $\Pi_2$
& 176.6 $\pm$ 25.9
& 83.0 $\pm$ 52.2
& 92.5 $\pm$ 31.3
& 70.6 $\pm$ 18.4
& 173.9 $\pm$ 19.3
& 58.4 $\pm$ 23.7\\
Dreamer
& $\Pi_2$
& 250.2 $\pm$ 9.6
& 3.0 $\pm$ 2.2
& 87.8 $\pm$ 16.1
& 87.3 $\pm$ 20.5
& 197.6 $\pm$ 24.6
& 10.0 $\pm$ 6.5\\
MAMBA
& $\Pi_3$
& 568.0 $\pm$ 229.1
& 475.7 $\pm$ 316.6
& 153.8 $\pm$ 34.7
& 121.1 $\pm$ 29.2
& 104.5 $\pm$ 25.7
& 99.9 $\pm$ 43.1\\
TAD-CE (Ours)
& $\Pi_3$
& \textbf{937.4 $\pm$ 9.8}
& \textbf{909.8 $\pm$ 21.9}
& \textbf{283.9 $\pm$ 16.2}
& \textbf{163.8 $\pm$ 53.0}
& \textbf{241.2 $\pm$ 36.8}
& \textbf{156.5 $\pm$ 129.6}\\
TAD-SC (Ours)
& $\Pi_3$
& \textbf{919.3 $\pm$ 21.9}
& \textbf{906.7 $\pm$ 21.7}
& 204.4 $\pm$ 62.2
& 143.9 $\pm$ 70.7
& 159.0 $\pm$ 36.9
& 104.3 $\pm$ 43.7\\
\bottomrule
\end{tabular}
\caption{Generalization performance (mean $\pm$ std) in DMC.
Numbers greater than 95$\%$ of the best performance are \textbf{bold}.}
\label{table_generalization}
\end{table*}

\begin{table*}[t!]
\centering
\tablestyle{3.5pt}{1.1}
\footnotesize
\begin{tabular}{cccccccccc}
\toprule
\multirow{2}*{Algorithms}& \multirow{2}*{Hypothesis} & Half-Cheetah-Fwd-Back(1e7) & \multicolumn{2}{c}{Half-Cheetah-Vel(1e7)} & \multicolumn{2}{c}{Humanoid-Direc-2D(1e6)}\\
& & Train\&Test & Train & Test & Train & Test\\
\midrule
PlaNet
& $\Pi_2$
& 30.5 $\pm$ 42.9
& -198.1 $\pm$ 1.9 
& -202.1 $\pm$ 1.8 
& 215.9 $\pm$ 72.3
& 220.6 $\pm$ 75.3 \\
Dreamer
& $\Pi_2$
& 127.4 $\pm$ 181.8 
& -151.4 $\pm$ 0.4
& -169.4 $\pm$ 1.2
& 260.5 $\pm$ 48.9
& 263.5 $\pm$ 52.3 \\
RL2(zero-shot)
& $\Pi_3$
& 1070.7 $\pm$ 109.7
& ---
& -70.3 $\pm$ 6.7
& ---
& 191.9 $\pm$ 50.8 \\
RL2(few-shot)
& $\Pi_3$
& 1006.9 $\pm$ 26.4
& ---
& -146.9 $\pm$ 0.4
& ---
& 268.8 $\pm$ 30.2 \\
MAML(few-shot)
& ---
& 429.3 $\pm$ 81.4
& ---
& -121.0 $\pm$ 37.1
& ---
& 205.3 $\pm$ 34.7 \\
VariBAD(zero-shot)
& $\Pi_3$
& 1177.5 $\pm$ 94.9
& ---
& -58.4 $\pm$ 20.6
& ---
& 260.3 $\pm$ 61.6 \\
TAD-CE (Ours)
& $\Pi_3$
& 1455.8 $\pm$ 78.3
& \textbf{-49.3 $\pm$ 1.9}
& \textbf{-47.1 $\pm$ 0.3}
& \textbf{339.5 $\pm$ 78.7}
& \textbf{335.5 $\pm$ 70.5} \\
TAD-SC (Ours)
& $\Pi_3$
& \textbf{1541.5 $\pm$ 114.8}
& -\textbf{50.5 $\pm$ 1.6}
& -\textbf{49.6 $\pm$ 1.6}
& 260.2 $\pm$ 185.0
& 249.0 $\pm$ 168.9\\
\bottomrule
\end{tabular}
\caption{Generalization performance (mean $\pm$ std) in MuJoCo. Numbers greater than 95$\%$ of the best performance are \textbf{bold}. 
}
\label{table_mujoco}
\end{table*}

\subsection{Experimental results for TDR}
\label{sec_exp_tdr}
To validate our analysis of TDR, we report the performance of different algorithms across different task combinations in Table~\ref{table_multi_task}.
The results show that TAD-CE and TAD-SC outperform all baselines, particularly in environments with high TDR.
As illustrated by Theorem~\ref{thm-2}, in task combinations with zero TDR, since different tasks share the same optimal action, policy hypothesis sets $\Pi_1,\Pi_2$ are sufficient to obtain the optimal policy. Therefore, baselines like CURL and Dreamer also perform well in such scenarios, e.g., Cartpole-balance\&balance$\_$sparse.
Conversely, in task combinations where the optimal Q-functions vary significantly across tasks, such as Cheetah-run and Cheetah-run$\_$back, the TDR is significantly huge. In these cases, methods with policies in $\Pi_1,\Pi_2$ (e.g., CURL and Dreamer) cannot differentiate between tasks and perform poorly. Consistent with Theorem~\ref{thm-2}, baselines leveraging $\Pi_3$ (e.g., MAMBA) outperform other baselines, and TAD further improves performance substantially.

Moreover, we provide qualitative visualizations to further illustrate how TAD operates.
Fig.~\ref{fig_imagine} presents future predictions generated by TAD for two tasks (Cheetah-run and Cheetah-flip). Using the same agent trained with TAD, we collect trajectories for both tasks and display them in the first and third rows, respectively. Given the first 7 steps as context, where the observations are visually similar while the rewards differ, we roll out the next 55 steps purely in imagination using the trained RIWM. The corresponding predicted trajectories are shown in the second and fourth rows of Figure~\ref{fig_imagine}. As observed, TAD produces distinct predictions across tasks solely based on the reward signals, indicating that the agent successfully infers task context from historical information. In addition, the accurate long-horizon predictions and high-quality reconstructions produced by RIWM highlight its potential as a foundation for training larger-scale world models.

In Figure~\ref{fig_walker}, we further visualize the latent states of world models trained with Dreamer, TAD-CE, and TAD-SC. Trajectories from different tasks are embedded into a two-dimensional space using t-SNE~\cite{52}. The latent representations from TAD-CE and TAD-SC form clearly separated clusters across tasks, whereas Dreamer produces overlapping embeddings and fails to distinguish task context. This visualization confirms that TAD effectively learns task-aware representations and differentiates between tasks at the representation level.

\begin{figure*}[t]
\centering
\includegraphics[width=0.9\linewidth]{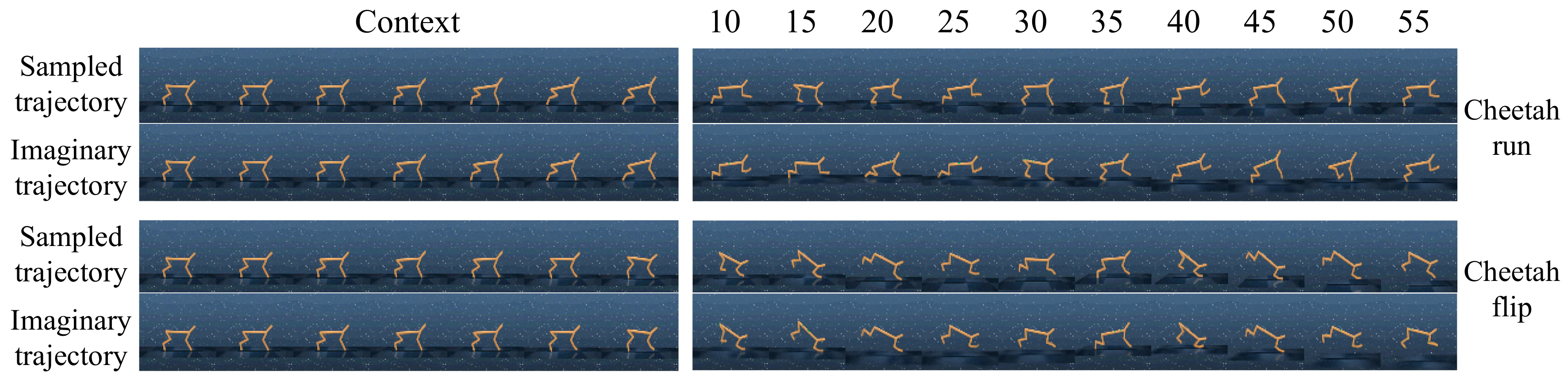}
\caption{
Sampled trajectories and imaginary trajectories of TAD for different tasks (Cheetah-run and Cheetah-flip).
}
\label{fig_imagine}
\end{figure*}

\subsection{Experimental results for task generalization}
\label{sec_exp_gene}

To address the second question, we report task generalization results on both image-based and state-based environments in Table~\ref{table_generalization}–\ref{table_mujoco}.
In Table~\ref{table_generalization}, TAD exhibits stronger generalization capability than all baselines, indicating that it can simultaneously handle multiple training tasks and effectively generalize to unseen test tasks. Additional qualitative visualizations are provided in Appendix C.3.

Furthermore, Table~\ref{table_mujoco} compares TAD with existing meta-RL methods in state-based environments, where the training horizons are $1e7$, $1e7$, and $1e6$, respectively. Many context-based meta-RL methods, such as RL2, PEARL, and VariBAD, utilize historical rewards and thus belong to $\Pi_3$. As a result, they are capable of distinguishing different tasks, which is consistent with our theoretical analysis.
As shown in Table~\ref{table_mujoco}, TAD achieves substantial performance gains on both training and test tasks, demonstrating that leveraging task-aware representations enables more effective generalization to unseen tasks.

\subsection{Ablation study}
\label{sec_exp_abla}

\paragraph{Reward signals.} In Fig.~\ref{fig_ablation}, we present ablation studies on the task-context term in TAD. In particular, we evaluate Dreamer(w/ r), which augments Dreamer to $\Pi_3$ by incorporating reward conditioning. The results show that simply providing reward signals can help Dreamer distinguish different tasks to some degree, which is consistent with our analysis in Sec.~\ref{theoretical}. Nevertheless, TAD-CE and TAD-SC achieve notably stronger performance, indicating that the proposed RIWM and the corresponding optimization method further enhance task generalization ability.

We also conduct ablations on reward signals for model-free methods such as CURL and SAC+AE, both of which belong to $\Pi_1$. Specifically, we design CURL (w/ r) and SAC+AE (w/ r) by directly appending the reward to the observation at each timestep. As shown in Table~\ref{table_multi_task}, their performance remains comparable to the original CURL and SAC+AE implementation. This is because these variants only use the reward and observation from the current timestep, without leveraging historical information, and thus still fail to effectively distinguish between different tasks.

\begin{figure}[t]
\centering
\subfloat[T-SNE visualization.\label{fig_walker}]{
    \centering
    \includegraphics[width=0.48\linewidth]{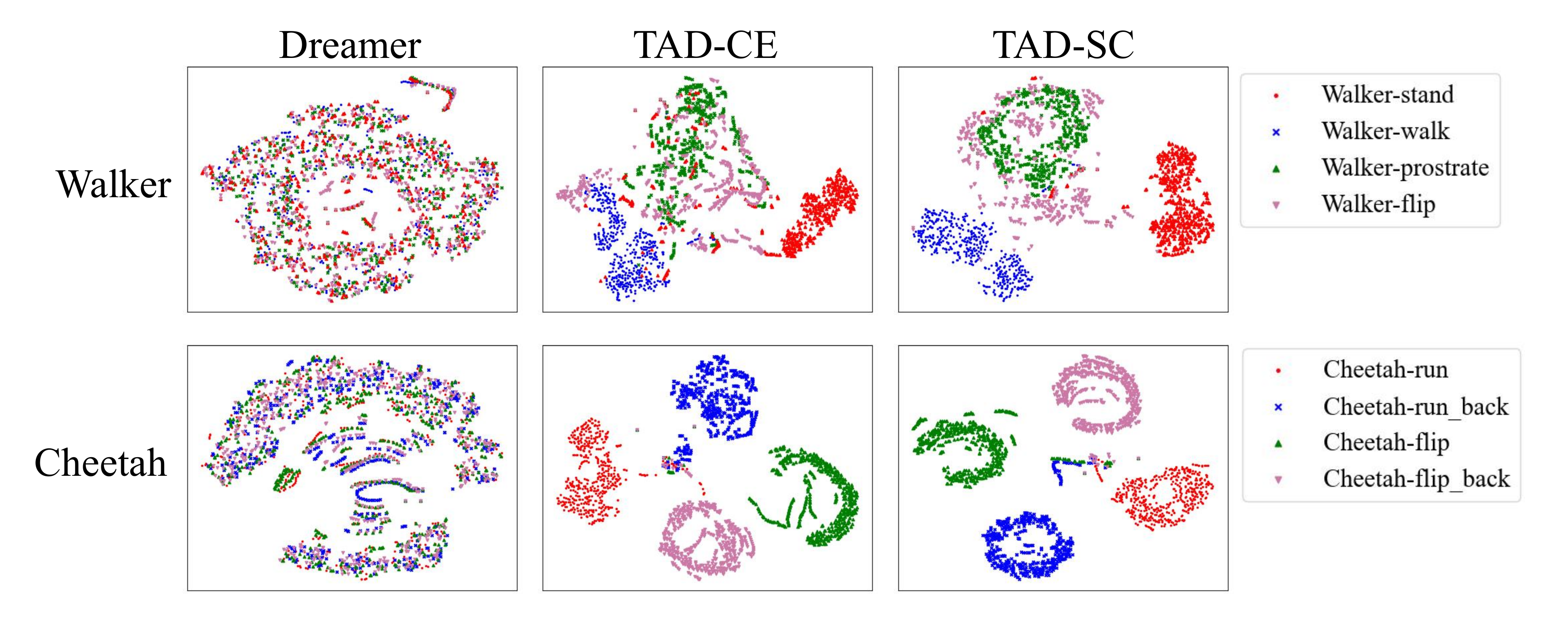}
}
\subfloat[Ablation study on Reward Signal.\label{fig_ablation}]{
    \centering
    \includegraphics[width=0.48\linewidth]{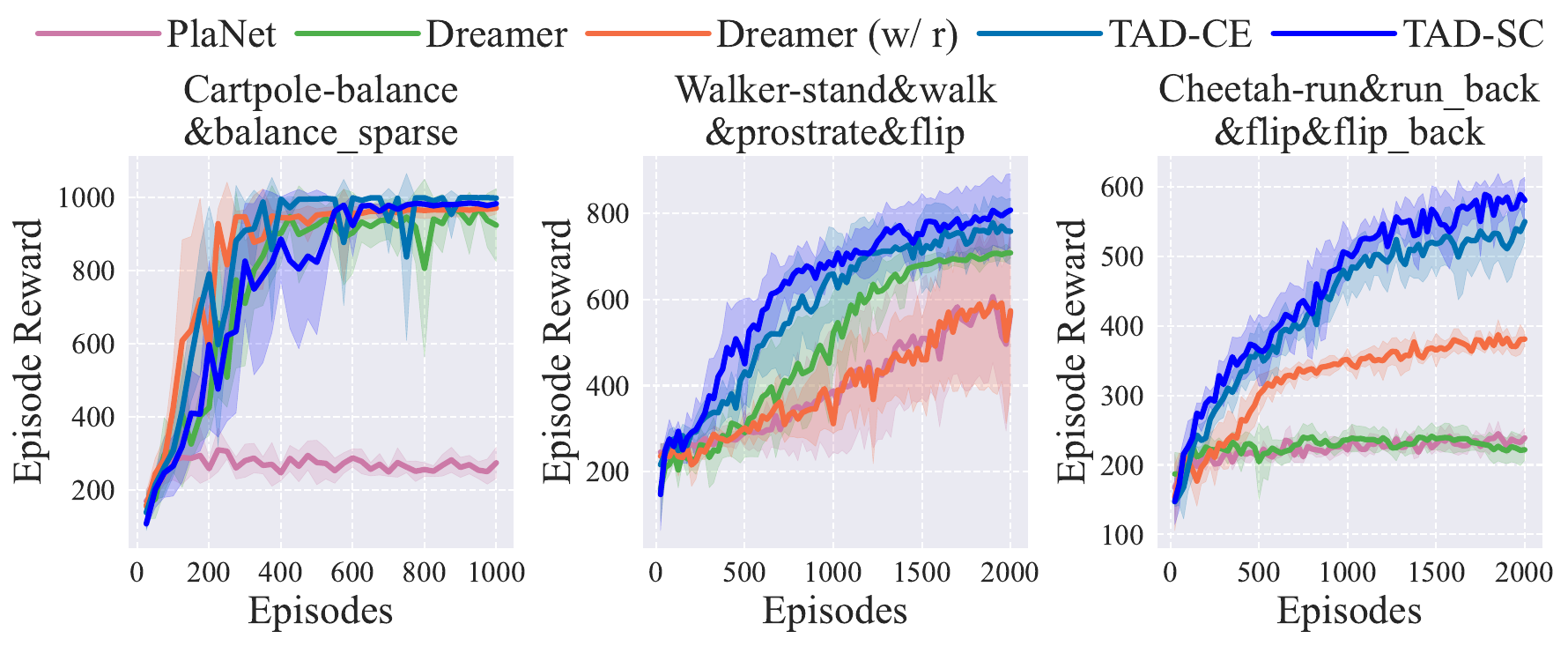}
}
\caption{(a) The t-SNE clustering of state embeddings for different tasks via Dreamer, TAD-CE, and TAD-SC.
(b) The ablation study on the reward signals.
}
\end{figure}

\paragraph{Extension to dynamics generalization.} To address the third question, we evaluate TAD in more general settings involving variations in observations, dynamics, and/or action spaces. Since TAD infers task information with all historical information, it can be directly applied to these settings without modification. We construct task distributions with different embodiments (Acrobot-Cartpole-Pendulum, Walker-Cheetah-Hopper), as well as different dynamics (Cheetah-run$\_$mass, Walker-walk$\_$mass).
More details and results are provided in Appendix E.2, where TAD demonstrates substantially stronger performance than the baselines, highlighting its potential for dynamics generalization and even cross-embodiment adaptation.

\subsection{Limitations and discussion}
Regarding limitations, TAD assumes that task context can be inferred from historical interaction data in order to generalize to unseen tasks.
Consequently, its effectiveness may be reduced in environments with extremely sparse rewards. In Appendix A.7, we further show that without extra prior knowledge or fine-tuning, zero-shot generalization to unseen tasks with extremely sparse rewards is fundamentally infeasible, as the agent lacks sufficient information to distinguish among tasks.
Nevertheless, in relatively sparse-reward settings, our experiments in Appendix E.3 demonstrate that TAD can still infer the current task and generalize effectively to unseen tasks.

\section{Conclusion}
In this work, we propose Task-Aware Dreamer (TAD), a novel framework that leverages full historical information to handle multiple tasks and employs Reward-Informed World Models (RIWM) to capture task-specific latent features. In TAD, we derive a variational lower bound on the data log-likelihood that incorporates a task context term, encouraging the world model to learn task-specific representations.
To justify the necessity of TAD’s components, we first introduce the Task Distribution Relevance (TDR) metric to characterize the relevance of different tasks from the task distribution. We then theoretically show that Markovian policies perform poorly in task distributions with high TDR. Extensive experiments in both image-based and state-based settings demonstrate that TAD substantially improves performance in the task distribution setting, particularly for task distributions with high TDR, and successfully generalizes to unseen tasks in a zero-shot manner.

\Acknowledgements{
This work was supported by the NSFC Projects
 (Nos. 92370124, 92248303, U25B6003,  62350080), Beijing Natural Science Foundation L247011, Tsinghua Institute for Guo Qiang, and the High Performance
 Computing Center, Tsinghua University. 
}


\begin{appendix}

\section{Proof of theorems}
In this section, we will provide detailed proofs of theorems.

\subsection{The proof of Theorem 1}

\begin{proof}
First, we will show that $\mathcal{H}_2 \subseteq \mathcal{H}_1$, following the proof in~\cite{39} that considers the setting where the rewards of all environments are the same.
\begin{equation}
\begin{split}
    \{q_m\} \in \mathcal{H}_2 
    \Leftrightarrow &\exists (p_m)_{m=1}^M: p_m(o_t^m, a_t^m) = o_{t+1}^m,\\
    & \exists (r_m)_{m=1}^M: r_m(o_t^m, a_t^m) = r_{t}^m, 
    q_m (o, a) = r_m(o,a) + \gamma \max_{a'} q_m (p_m(o, a), a'), \forall o,a  \\
    \Rightarrow & \exists (r_m)_{m=1}^M:: p_m(o_t^m, a_t^m) = o_{t+1}^m, \\
    & \exists (r_m)_{m=1}^M: r_m(o_t^m, a_t^m) = r_{t}^m, 
    q_m (o_t^m, a_t^m) = r^m_t + \gamma \max_{a'} q_m (p_m(o_t^m, a_t^m), a'), \forall m, t \\
    \Rightarrow & q_m (o_t^m, a_t^m) = r^m_t + \gamma \max_{a'} q_m (o_{t+1}^m, a'), \forall m, t \\
    \Leftrightarrow & \{q_m\} \in \mathcal{H}_1,
\end{split}
\end{equation}
thus we have $\mathcal{H}_2\subseteq\mathcal{H}_1$. Next, we prove that $\mathcal{H}_3\subseteq\mathcal{H}_2$, which is mainly because we can utilize similar dynamic structures from different tasks to narrow down the hypothesis spaces of the dynamic model.
\begin{equation}
\begin{split}
    \{q_m\} \in \mathcal{H}_3
    \Leftrightarrow & \exists p: p(o_t^m, a_t^m) = o_{t+1}^m, \forall m,t \\
    & \exists (r_m)_{m=1}^M: r_m(o_t^m, a_t^m) = r_{t}^m, 
    q_m (o, a) = r_m(o, a) + \gamma \max_{a'} q_m (p(o, a), a'), \forall m,o,a \\
    \Rightarrow & \exists (p_m)_{m=1}^M : p_m(o_t^m, a_t^m) = o_{t+1}^m, \forall m,t \\
    & \exists (r_m)_{m=1}^M: r_m(o_t^m, a_t^m) = r_{t}^m,
    q_m (o, a) = r_m(o, a) + \gamma \max_{a'} q_m (p_m(o, a), a'), \forall m,o,a \\
    \Leftrightarrow & \{q_m\} \in \mathcal{H}_2,
\end{split}
\end{equation}
thus we have $\mathcal{H}_3\subseteq\mathcal{H}_2$.
\end{proof}

\subsection{The derivation of the ELBO}
\label{pf_elbo}
We use $\textcolor{mydarkgreen}{q}$ to represent $q(s_{1:T}|o_{1:T},a_{1:T},r_{1:T})$, ${\hat{q}}$ to present $q(s_{t}|o_{\leq t},a_{< t},r_{< t})$, $\textcolor{red}{\tilde{q}}$ to represent $q(s_{t-1}|o_{\leq (t-1)}, r_{< (t-1)}, a_{< (t-1)})$, and we have
\begin{equation}
\begin{split}
    & \ln p(o_{1:T},r_{1:T}, \mathcal{M} |a_{1:T}) \\
    = &\ln \mathbb{E}_{p(s_{1:T}|a_{1:T})} \left[p(o_{1:T},r_{1:T}, \mathcal{M} |s_{1:T})\right]\\
    = &\ln \mathbb{E}_{p(s_{1:T}|a_{1:T})} \left[p(o_{1:T},r_{1:T}|\mathcal{M}, s_{1:T}) p( \mathcal{M} | s_{1:T})\right]\\
    = &\ln \mathbb{E}_{p(s_{1:T}|a_{1:T})}\left[p( \mathcal{M} | s_{1:T}) \prod_{t=1}^T p(o_t,r_t|\mathcal{M}, s_t)\right]\\
    = &\ln \mathbb{E}_{\textcolor{mydarkgreen}{q}}\left[p( \mathcal{M} | s_{1:T}) \prod_{t=1}^T p(o_t,r_t|\mathcal{M},s_t)\frac{p(s_t|s_{t-1}, a_{t-1})}{q(s_t|o_{\leq t}, r_{< t}, a_{< t})}\right] \\
    \ge & \mathbb{E}_{\textcolor{mydarkgreen}{q}} [\ln p( \mathcal{M} | s_{1:T}) + \sum_{t=1}^T \ln p(o_t,r_t|\mathcal{M},s_t)]
    + \mathbb{E}_{\textcolor{mydarkgreen}{q}}\sum_{t=1}^T\left[ \ln  p(s_t|s_{t-1}, a_{t-1})- \ln  q(s_t|o_{\leq t}, r_{< t}, a_{< t})\right]\\
    = & \mathbb{E}_{\textcolor{mydarkgreen}{q}} [\ln p( \mathcal{M} | s_{1:T})]
    + \sum_{t=1}^T\left[  \mathbb{E}_{\textcolor{mydarkgreen}{q}}\ln p(o_t,r_t|\mathcal{M},s_t) - \mathbb{E}_{\textcolor{mydarkgreen}{q}} \ln \frac{q(s_t|o_{\leq t}, r_{< t}, a_{< t})}{ p(s_t|s_{t-1}, a_{t-1})} \right]\\
    = &\mathbb{E}_{\textcolor{mydarkgreen}{q}} [\ln p( \mathcal{M} | s_{1:T})] + \sum_{t=1}^T \mathbb{E}_{{\hat{q}}} [\ln p(o_t,r_t|\mathcal{M},s_t)]
    - \sum_{t=1}^T \mathbb{E}_{{\hat{q}} \textcolor{red}{\tilde{q}}} \left[\ln \frac{q(s_t|o_{\leq t}, r_{< t}, a_{< t})}{ p(s_t|s_{t-1}, a_{t-1})}\right] \\
    = &\mathbb{E}_{\textcolor{mydarkgreen}{q}}[\ln p( \mathcal{M} | s_{1:T})] + \sum_{t=1}^T \mathbb{E}_{{\hat{q}}}[\ln p(o_t,r_t|\mathcal{M},s_t)]
    - \sum_{t=1}^T \mathbb{E}_{\textcolor{red}{\tilde{q}}} [\mathrm{KL} \left(q(s_t|o_{\leq t}, r_{< t}, a_{< t})\| p(s_t|s_{t-1}, a_{t-1})\right)].
\end{split}
\end{equation}

Thus, we have proven the ELBO.

\subsection{The proof of Theorem 2}
\label{pf_thm1}
\begin{proof}
We first prove that for $\forall \pi\in\Pi_2$, we have $\mathbb{E}_{\mathcal{M}\sim\mathcal{T}} [J_{\mathcal{M}}(\pi)] = J_{\bar{\mathcal{M}}}(\pi)$.

For any $\mathcal{M}=(\mathcal{S}, \mathcal{A}, \mathcal{P}, \mathcal{R}_{\mathcal{M}}, \gamma)\sim\mathcal{T}$, we can use the policy $\pi$ to interact with $\mathcal{M}$ and get the trajectory $\tau=(s_0^{\mathcal{M}}, a_0^{\mathcal{M}}, r_1^{\mathcal{M}}, s_1^{\mathcal{M}}, a_1^{\mathcal{M}}, r_2^{\mathcal{M}}, ...)$. Since the dynamic transition $\mathcal{P}$ is the same for all $\mathcal{M}$ and the policy $\pi\in\Pi_2$ only depends on historical states and actions, we naturally have that the distribution of all states and actions $(s_0^{\mathcal{M}}, a_0^{\mathcal{M}}, s_1^{\mathcal{M}}, a_1^{\mathcal{M}}, ...)$ are the same for all $\mathcal{M}\sim\mathcal{T}$ as well as $\bar{\mathcal{M}}$. Consequently, we have 
\begin{equation}
\begin{split}
    \mathbb{E}_{\mathcal{M}\sim\mathcal{T}} [J_{\mathcal{M}}(\pi)] 
    =& \mathbb{E}_{\mathcal{M}\sim\mathcal{T}} \mathbb{E}_{\tau\sim \mathcal{P},\pi} [R_{\mathcal{M}}(\tau)] 
    = \mathbb{E}_{\mathcal{M}\sim\mathcal{T}} \mathbb{E}_{\tau\sim \mathcal{P},\pi} \left[\sum_{t=0}^{\infty} \gamma^t r_t^{\mathcal{M}}\right] 
    = \mathbb{E}_{\tau\sim \mathcal{P},\pi} \left[\sum_{t=0}^{\infty} \gamma^t \mathbb{E}_{\mathcal{M}\sim\mathcal{T}} [r_t^{\mathcal{M}}]\right] \\
    =& \mathbb{E}_{\tau\sim \mathcal{P},\pi} \left[\sum_{t=0}^{\infty} \gamma^t \mathbb{E}_{\mathcal{M}\sim\mathcal{T}} [\mathcal{R}_{\mathcal{M}}(s_t^{\mathcal{M}}, a_t^{\mathcal{M}})]\right] 
    = \mathbb{E}_{\tau\sim \mathcal{P},\pi} \left[\sum_{t=0}^{\infty} \gamma^t [\bar{\mathcal{R}}(s_t^{\mathcal{M}}, a_t^{\mathcal{M}})]\right] \\
    =& \mathbb{E}_{\tau\sim \mathcal{P},\pi} [R_{\bar{\mathcal{M}}}(\tau)]
    = J_{\bar{\mathcal{M}}}(\pi).
\end{split}
\end{equation}
It is well known that the optimal policy in single MDP is memory-less, i.e.,  $\max_{\pi\in\Pi_2} J_{\bar{\mathcal{M}}}(\pi) = \max_{\pi\in\Pi_1} J_{\bar{\mathcal{M}}}(\pi)$. Consequently, we have
\begin{equation}
\begin{split}
    J_{\mathcal{T}}^2 &= \max_{\pi\in\Pi_2}\mathbb{E}_{\mathcal{M}\sim\mathcal{T}} [J_{\mathcal{M}}(\pi)] = \max_{\pi\in\Pi_2}  J_{\bar{\mathcal{M}}}(\pi)
    = \max_{\pi\in\Pi_1} J_{\bar{\mathcal{M}}}(\pi) = \max_{\pi\in\Pi_1}\mathbb{E}_{\mathcal{M}\sim\mathcal{T}} [J_{\mathcal{M}}(\pi)] = J_{\mathcal{T}}^1 \\
    &\leq \mathbb{E}_{\mathcal{M}\sim\mathcal{T}} \left[\max_{\pi\in\Pi_1} J_{\mathcal{M}}(\pi)\right] = J_{\mathcal{T}}^*.
\end{split}
\end{equation}
Thus we have proven this result.
\end{proof}

\subsection{The Proof of Theorem 3}
\label{pf_thm2}

\begin{proof}
Our proof follows some previous work~\cite{53,54}. First, we consider the bellman equation of the value function of $\pi,\pi_{\mathcal{M}}^*\in\Pi_1$ in $\mathcal{M}$ as
\begingroup
\small
\begin{equation*}
\begin{split}
V_{\mathcal{M}, \pi}(s) =& \sum_a \pi(a|s)\left[\mathcal{R}(s, a) + \gamma \sum_{s'}\mathcal{P}(s'|s, a)V_{\mathcal{M}, \pi}(s')\right],\\
    V_{\mathcal{M}, \pi_{\mathcal{M}}^*}(s) 
    =& \sum_a \pi_{\mathcal{M}}^*(a|s)\left[\mathcal{R}(s, a) + \gamma \sum_{s'}\mathcal{P}(s'|s, a)V_{\mathcal{M}, \pi_{\mathcal{M}}^*}(s')\right].
\end{split}
\end{equation*}
\endgroup
Defining $\Delta V(s) \triangleq V_{\mathcal{M}, \pi}(s) - V_{\mathcal{M}, \pi_{\mathcal{M}}^*}(s)$ as the difference of these two value functions, we can further deduce that
\begin{equation}
\begin{split}
\label{15}
    &V_{\mathcal{M}, \pi}(s) - V_{\mathcal{M}, \pi_{\mathcal{M}}^*}(s) \\
    =& \gamma\sum_a \Delta\pi(a|s) \sum_{s'}\mathcal{P}(s'|s, a)V_{\mathcal{M}, \pi_{\mathcal{M}}^*}(s') 
    + \gamma\sum_a \pi(a|s) \sum_{s'}\mathcal{P}(s'|s, a)\Delta V(s') 
    + \sum_a \Delta\pi(a|s)\mathcal{R}(s, a)\\
    =& \sum_a \Delta\pi(a|s) Q_{\mathcal{M}, \pi_{\mathcal{M}}^*}(s,a) 
    + \gamma\sum_a \pi(a|s) \sum_{s'}\mathcal{P}(s'|s, a)\Delta V(s'),
\end{split}
\end{equation}
here $\Delta\pi(a|s)= \pi(a|s) - \pi_{\mathcal{M}}^*(a|s)$. Since Eq.~\eqref{15} holds for any $s$, thus we calculate its expectation for $s\sim d_{\mathcal{M}}^{\pi_{\mathcal{M}}^*}$:
\begin{equation}
\begin{split}
    \sum_s d_{\mathcal{M}}^{\pi}(s) \Delta V(s) 
    =&\sum_s d_{\mathcal{M}}^{\pi}(s) [V_{\mathcal{M}, \pi}(s)-V_{\mathcal{M}, \pi_{\mathcal{M}}^*}(s)]\\
    =& \sum_s d_{\mathcal{M}}^{\pi}(s)\sum_a \Delta\pi(a|s) Q_{\mathcal{M}, \pi_{\mathcal{M}}^*}(s,a)
    + \gamma\sum_s d_{\mathcal{M}}^{\pi}(s)\sum_a \pi(a|s)\sum_{s'}\mathcal{P}(s'|s, a)\Delta V(s')\\
    =& \sum_s d_{\mathcal{M}}^{\pi}(s)\sum_a \Delta\pi(a|s) Q_{\mathcal{M}, \pi_{\mathcal{M}}^*}(s,a)
    + \sum_{s'}\Delta V(s')\left[\gamma\sum_s d_{\mathcal{M}}^{\pi}(s)\sum_a \pi(a|s) \mathcal{P}(s'|s, a)\right].
\end{split}
\end{equation}
Since 
$\gamma\sum_{s'}d_{\mathcal{M}}^{\pi}(s')\sum_{a}\pi(a|s')\mathcal{P}(s|s', a) = d_{\mathcal{M}}^{\pi}(s) - (1-\gamma)\mathcal{P}(s_0=s)$,
we have
\begin{equation}
\begin{split}
\label{16}
    \sum_s d_{\mathcal{M}}^{\pi}(s) &\Delta V(s)
    = \sum_s d_{\mathcal{M}}^{\pi}(s)\sum_a \Delta\pi(a|s) Q_{\mathcal{M}, \pi_{\mathcal{M}}^*}(s,a) 
    + \sum_{s'}\Delta V(s') \left[d_{\mathcal{M}}^{\pi}(s') - (1-\gamma)\mathcal{P}(s_0=s')\right].
\end{split}
\end{equation}
By moving the second term of the right part in Eq.~\eqref{16} to the left part, we can deduce that
\begin{equation}
\begin{split}
\label{app_eq_10}
    (1-\gamma)\sum_{s'}\Delta V(s') \mathcal{P}(s_0=s') 
    =&  \sum_s d_{\mathcal{M}}^{\pi}(s)\sum_a \Delta\pi(a|s) Q_{\mathcal{M}, \pi_{\mathcal{M}}^*}(s,a),
\end{split}
\end{equation}
thus we can calculate that
\begin{equation}
\begin{split}
\label{app_eq_11}
    J_{\mathcal{M}}(\pi) - J_{\mathcal{M}}(\pi_{\mathcal{M}}^*) 
    = & \sum_{s'}\Delta V(s') \mathcal{P}(s_0=s')
    = \frac{1}{1-\gamma} \sum_s d_{\mathcal{M}}^{\pi}(s)\sum_a \Delta\pi(a|s) Q_{\mathcal{M}, \pi_{\mathcal{M}}^*}(s,a)\\
    =& \frac{1}{1-\gamma} \sum_s d_{\mathcal{M}}^{\pi}(s)\sum_a [\pi(a|s) - \pi_{\mathcal{M}}^*(a|s)] Q_{\mathcal{M}, \pi_{\mathcal{M}}^*}(s,a)\\
    =& \frac{1}{1-\gamma} \mathbb{E}_{s\sim d_{\mathcal{M}}^{\pi}} \mathbb{E}_{a\sim\pi(\cdot|s)} \left(1- \frac{\pi_{\mathcal{M}}^*(a|s)}{\pi(a|s)}\right) Q_{\mathcal{M}, \pi_{\mathcal{M}}^*}(s,a)\\
    =& \frac{1}{1-\gamma} \mathbb{E}_{s\sim d_{\mathcal{M}}^{\pi}} \int_{\mathcal{A}}\pi(a|s) \left(1- \frac{\pi_{\mathcal{M}}^*(a|s)}{\pi(a|s)}\right) Q_{\mathcal{M}}^*(s, a) \mathrm{d} a\\
    =& \frac{1}{1-\gamma} \mathbb{E}_{s\sim d_{\mathcal{M}}^{\pi}} \left[\int_{a}\pi(a|s)Q_{\mathcal{M}}^*(s,a)da - \max_a Q_{\mathcal{M}}^*(s,a)\right].
\end{split}
\end{equation}
Consequently, we have
\begin{equation}
\begin{split}
\label{app_eq_12}
    \mathbb{E}_{\mathcal{M}\sim\mathcal{T}}\left[J_{\mathcal{M}}(\pi_{\mathcal{M}}^*) - J_{\mathcal{M}}(\pi)\right] 
    =& \frac{1}{1-\gamma}\mathbb{E}_{\mathcal{M}\sim\mathcal{T}} \mathbb{E}_{s\sim d_{\mathcal{M},\pi}(\cdot)} \left[\max_a Q_{\mathcal{M}}^*(s,a) - \int_{a}\pi(a|s)Q_{\mathcal{M}}^*(s,a)da\right]\\
    \ge & \frac{1}{1-\gamma}\mathbb{E}_{s\sim d_{\mathcal{M},\pi}(\cdot)} \left[\mathbb{E}_{\mathcal{M}\sim\mathcal{T}}\max_a Q_{\mathcal{M}}^*(s,a)  - \max_{a}\mathbb{E}_{\mathcal{M}\sim\mathcal{T}}Q_{\mathcal{M}}^*(s,a)\right] \\
    = & \frac{1}{1-\gamma}\mathbb{E}_{s\sim d_{\mathcal{M},\pi}} \left[D_{\text{TDR}}(\mathcal{T}, s)\right].
\end{split}
\end{equation}
Since $J_{\mathcal{M}}(\pi_{\mathcal{M}}^*) = \max_{\pi\in\Pi_1} J_{\mathcal{M}}(\pi)$, we have
\begin{equation}
\begin{split}
    \mathbb{E}_{\mathcal{M}\sim\mathcal{T}} \max_{\pi\in\Pi_1} J_{\mathcal{M}}(\pi) - \max_{\pi\in\Pi_1} \mathbb{E}_{\mathcal{M}\sim\mathcal{T}}  J_{\mathcal{M}}(\pi) 
    \ge& \frac{1}{1-\gamma}\mathbb{E}_{s\sim d_{\mathcal{M},\pi^*}} \left[D_{\text{TDR}}(\mathcal{T}, s)\right],
\end{split}
\end{equation}

Thus we have proven this result.
\end{proof}

\subsection{Expressive ability of $\Pi_3$}
\begin{proposition}
For $\forall\epsilon_1,\epsilon_2$ satisfying $0<\epsilon_1\leq 1,0<\epsilon_2\leq 1$, there exists a task distribution $\mathcal{T}$ satisfying that
\begin{equation}
\begin{split}
    J_{\mathcal{T}}^1=J_{\mathcal{T}}^2 \leq \epsilon_1,~~J_{\mathcal{T}}^3\ge1-\epsilon_2,~~J_{\mathcal{T}}^*=1.
\end{split}
\end{equation}
\end{proposition}
\begin{proof}
Given fixed discount factor $\gamma\in(0,1)$, we first take $n\in\mathbb{N}$ satisfying $n \ge \frac{1}{\epsilon_1}$, $A = \epsilon_2\frac{n}{n-1}\frac{1-\gamma}{1-\gamma^n}$, and $B=\frac{1-\gamma}{\gamma^{n+1}}\left(1-\epsilon_2\frac{n}{n-1}\right)$. We construct state sets $\mathcal{S} = \{s_t^l\} (t=0,1,...,\infty,l=1,...,n)$ and action sets $\mathcal{A} = \{a_j\}_{j=1}^n$. Then we construct $n$ tasks $\mathcal{M}_i = (\mathcal{S}, \mathcal{A}, \mathcal{P}, \mathcal{R}_i, \gamma), i=1,...,n$, which share the same dynamic $\mathcal{P}$ and different reward functions $\mathcal{R}_i$. The initial state of each task is $s_0^1$, and the dynamic as well as reward functions are as below
\begin{equation}
\begin{split}
    &\mathcal{P}(s_t^l, a_j) = \mathbb{I}(s=s_{t+1}^j), \quad j,l=1,...,n;\ t=0,1,...,\infty\\
    &\mathcal{R}_i(s_t^l, a_j) = f(t)\mathbb{I}(i=j),\quad i,j,l=1,...,n;\ t=0,1,...\infty \\
& f(t) = \left \{
\begin{array}{ll}
    A,\quad     & t\leq n-1\\
    B,\quad     & t\ge n\\
\end{array}
\right.
\end{split}
\end{equation}
Take $\mathcal{T}$ as the uniform distribution over $\mathcal{M}_1, ..., \mathcal{M}_n$, thus we have
\begin{equation}
\begin{split}
    J_{\mathcal{M}_i}^* =& A + A\gamma + ...+A\gamma^{n-1}+B\gamma^{n}+B\gamma^{n+1}+... 
    = A\frac{1-\gamma^{n}}{1-\gamma} + B \frac{\gamma^{n+1}}{1-\gamma} = 1 \\
    J_{\mathcal{T}}^* 
    =& \frac{1}{n} \sum_{i=1}^n J_{\mathcal{M}_i}^* = 1.\\
\end{split}
\end{equation}
Since our construction satisfies $\mathbb{E}_{\mathcal{T}} [\mathcal{R}_i(s_k^l, a_j)] = \frac{f(k)}{n}$, for $\forall \pi\in \Pi_2$, we have
\begin{equation}
\begin{split}
    J_{\mathcal{T}}(\pi) =& \frac{1}{n} (A + A\gamma + ...+A\gamma^{n-1}+B\gamma^{n}+B\gamma^{n+1}+... )
    = \frac{1}{n}, \\
    J_{\mathcal{T}}^1 =& J_{\mathcal{T}}^2 = \max_{\pi\in\Pi_2} J_{\mathcal{T}}(\pi) =  \frac{1}{n} \leq \epsilon_1.
\end{split}
\end{equation}
Moreover, we can construction an agent $\hat\pi\in\Pi_3$ that takes action via the historical trajectory $\hat{\tau}_t = (\hat{s}_0, \hat{a}_0, \hat{r}_0, ...\hat{s}_t)$:
\begin{equation}
\begin{split}
    \hat{\pi}(a_j|\hat{\tau}_t) = & \mathbb{I}(j=t+1),\quad t=0,1,...,n-1 \\
    \hat{\pi}(a_j|\hat{\tau}_t) = & \mathbb{I}(j=i),\quad t=n,...,\infty \\
\end{split}
\end{equation}
here $i = \arg\max\{\hat{r}_0, \hat{r}_1, ..., \hat{r}_{n-1}\} + 1$, thus we have
\begin{equation}
\begin{split}
    J_{\mathcal{T}}^3
    \ge J_{\mathcal{T}}(\hat{\pi})
    =& \frac{1}{n} (A + A\gamma + ...+A\gamma^{n-1})+B\gamma^{n}+B\gamma^{n+1}+... \\
    =& \frac{A}{n}\frac{1-\gamma^{n}}{1-\gamma} + B \frac{\gamma^{n+1}}{1-\gamma}
    = 1 - \frac{(n-1)A}{n}\frac{1-\gamma^{n}}{1-\gamma}
    \ge 1-\epsilon_2.
\end{split}
\end{equation}
Thus we have proven this result.
\end{proof}

\subsection{Proof and discussion of Theorem 4}
In this part, we introduce an informed version of Theorem 4, about why optimizing $p(\mathcal{M}|l)\forall l\in \mathcal{L}$ is beneficial for task generalization, with detailed proofs. Recall that we consider the policy hypothesis $\mathcal{H}_3$ here, that each policy $\pi:\mathcal{L}\rightarrow \Delta(\mathcal{A}), \mathcal{L} = \cup_{t=1}^{\infty}\mathcal{L}_t, \mathcal{L}_t = (\mathcal{S}\times \mathcal{A}\times \mathbb{R})^{t-1} \times \mathcal{S}$. As directly such $\mathcal{S}-\mathcal{A}-\mathcal{R}$ memorized policy is difficult, we consider an alternative MDP as $\tilde{\mathcal{M}} = (\mathcal{L}, \mathcal{A}, \mathcal{P}_{\mathcal{M}}, \mathcal{R}_{\mathcal{M}}, \gamma)$. For $\forall l=(s_1,a_1,r_1,...,s_t)\in\mathcal{L}_t\subseteq\mathcal{L}$, we can sample the action $a_t$ from the distribution $\pi(\cdot|l)$. Then the environment will feedback the reward signal $r_t = \mathcal{R}_{\mathcal{M}}(l,a_t) = \mathcal{R}_{\mathcal{M}}(s_t, a_t)$, we can sample $s_{t+1}$ from the distribution $\mathcal{P}(\cdot|s_t,a_t)$, and the next environment state will be $l'=(s_1,a_1,r_1,...,s_t,a_t,r_t,s_{t+1}) \in \mathcal{L}_{t+1}$. In summary, we set the distribution $p(l' |l,a_t)$ as the corresponding dynamic $\mathcal{P}_{\mathcal{M}}$. Notice that although all $\mathcal{M}\in\mathcal{T}$ share the same dynamic $\mathcal{P}$, their new dynamic $\mathcal{P}_{\mathcal{M}}$ are different since the dynamic is related to the given reward. An obvious advantage of introducing $\tilde{\mathcal{M}}$ is that $\pi\in\Pi_3$ is now Markovian in $\tilde{\mathcal{M}}$, and it is much easier to analyze its performance.

Obviously, we can set $J_{\mathcal{M}}(\pi) = J_{\tilde{\mathcal{M}}}(\pi), Q_{\mathcal{M}}^*(s,a) = Q_{\tilde{\mathcal{M}}}^*(s,a) $ to simplify the notation, and we can prove that
\begin{theorem}
For any policy $\pi\in\Pi_3$, we have
\begin{equation}
\begin{split}
\label{app_ep_20}
    J_{\mathcal{T}}^* - \mathbb{E}_{\mathcal{M}\sim\mathcal{T}}\left[J_{\mathcal{M}}(\pi)\right]
    =& \frac{1}{1-\gamma} \int_{\mathcal{L}} p(l) \left[\int  p(\mathcal{M}|l) \max_a Q_{\mathcal{M}}^*(l,a) d\mathcal{M} - \int_{a,\mathcal{M}} \pi(a|l)  p(\mathcal{M}|l) Q_{\mathcal{M}}^*(l,a)dad\mathcal{M} \right] dl\\
    \ge & \frac{1}{1-\gamma} \int_{\mathcal{L}} p(l) \left[\int  p(\mathcal{M}|l) \max_a Q_{\mathcal{M}}^*(l,a) d\mathcal{M}
    - \max_a \int_{\mathcal{M}}   p(\mathcal{M}|l) Q_{\mathcal{M}}^*(l,a)d\mathcal{M} \right] dl,
\end{split}
\end{equation}
here $p(l)$ is a distribution of $\mathcal{L}$ related to $\mathcal{T},\pi$ and $p(\mathcal{M}|l)$ is the task posterior related to $\pi$.
\end{theorem}

\begin{proof}
$\forall\pi\in\mathcal{H}_3$, as $\pi$ is Markovian in $\tilde{\mathcal{M}}$, we can directly utilize the proof of Theorem 3 from the beginning to Eq.~\eqref{app_eq_11}, and the only difference is that the dynamics in $\tilde{\mathcal{M}}$ are different but the dynamics in $\mathcal{M}$ are the same. Thus, we need to change the Eq.~\eqref{app_eq_12} as
\begin{equation}
\begin{split}
    \mathbb{E}_{\mathcal{M}\sim\mathcal{T}}\left[J_{\mathcal{M}}(\pi_{\mathcal{M}}^*) - J_{\mathcal{M}}(\pi)\right] 
    =& \frac{1}{1-\gamma}\mathbb{E}_{\mathcal{M}\sim\mathcal{T}} \mathbb{E}_{l\sim d_{\tilde{\mathcal{M}},\pi}(\cdot)} \left[\max_a Q_{\mathcal{M}}^*(l,a)
     - \int_{a}\pi(a|l)Q_{\mathcal{M}}^*(l,a)da\right]\\
    =& \frac{1}{1-\gamma}\int p(\mathcal{M}) \int_{\mathcal{L}} d_{\tilde{\mathcal{M}},\pi}(l) \left[\max_a Q_{\mathcal{M}}^*(l,a)
    - \int_{a}\pi(a|l)Q_{\mathcal{M}}^*(l,a)da\right]\\
    =& \frac{1}{1-\gamma}\int_{\mathcal{L}} p(l) \int p(\mathcal{M}|l) \left[\max_a Q_{\mathcal{M}}^*(l,a)
     - \int_{a}\pi(a|l)Q_{\mathcal{M}}^*(l,a)da\right]\\
    \ge & \frac{1}{1-\gamma} \int_{\mathcal{L}} p(l) \left[\int  p(\mathcal{M}|l) \max_a Q_{\mathcal{M}}^*(l,a) d\mathcal{M}
    - \max_a \int_{\mathcal{M}}  Q_{\mathcal{M}}^*(l,a)d\mathcal{M} \right] dl,
\end{split}
\end{equation}
here $p(l) = \int p(\mathcal{M})d_{\tilde{\mathcal{M}},\pi}(l) d\mathcal{M}$, and $p(\mathcal{M}|l) =p(\mathcal{M})d_{\tilde{\mathcal{M}},\pi}(l)/p(l)$ is the posterior distribution. 
\end{proof}
Finally, we will show that maximizing $p(\mathcal{M}|l)$ is helpful for task generalization. In the training stage, we will sample a task $\mathcal{M}$ and corresponding state-action-reward pairs $l$, thus optimizing $p(\mathcal{M}|l)$ will make it to be closer to some Dirac distribution. In such a situation, for each $l$, we can infer a ``most possible" posterior task $\mathcal{M}_l$ with high $p(\mathcal{M}_l|l)$, thus we can approximately take $\pi(l) = \arg\max_a Q_{\mathcal{M}_l}^*(l,a)$ and the optimal gap calculated by Eq.~\eqref{app_ep_20} will be controlled. 

\subsection{Sparse reward tasks}
In this part, we show that generalizing to tasks with the same dynamics and sparse rewards without extra knowledge (like context) is extremely difficult and sometimes impossible. It is because we cannot distinguish them via historical information. Here we construct an example.

Assume there are $n$ MDPs, $\mathcal{M}_i (i=1,...,n)$, each MDP share the same state set $\mathcal{S}=\{s_1,...,s_T\}$ and action set $\mathcal{A}=\{a_1,...,a_n\}$. The initial state is $s_1$ and the dynamic is that $\mathcal{P}(s_{t+1}|s_t,a_i) = 1(t=1,...,T-1, i=1,...,n), \mathcal{P}(s_{T}|s_T,a_i) = 1(i=1,...,n)$. As for the reward function, we define that
\begin{equation}
    \mathcal{R}_{\mathcal{M}_i}(s_{T-1}, a_i) = 1, i =1,...,n.
\end{equation}
and the reward function is 0 otherwise. In this case, any policies (including Markovian, state-action memorized, and state-action-reward memorized) in this task distribution can only handle one task since they can not distinguish them.

\section{Pseudo code of TAD}
\label{appendix_algo}
The detailed pseudo code of TAD is provided in Algorithm~\ref{tad-algo}.

\begin{algorithm*}[h]
    \caption{Task Aware Dreamer (TAD)} 
    \begin{algorithmic}[1] 
        \REQUIRE $M$ training tasks $\{\mathcal{M}_m\}_{m=1}^M$, $M$ replay buffers $\{\mathcal{D}_m\}_{m=1}^M$, $N$ test tasks $\{\mathcal{M}_{M+n}\}_{n=1}^N$, initialize neural network parameters of world models, the policy, and the critic
        \WHILE{not converge}
        \STATE $\verb|//|  Model\ Training$
        \FOR{$\text{update step} = 1,2,..., U$}
        \STATE Sample observation-action-reward pairs form each replay buffer $\{(o_t^i,a_t^i,r_t^i)_{t=1}^{T}\} \sim \mathcal{D}_i, i=1,2,...,M$
        \STATE Calculate the deterministic state $h$ and further calculate model states $s$.
        \STATE Update the world models via optimizing Eq. (6).
        \STATE Collect imagined trajectories from each $s_t$ via the policy and the world models.
        \STATE Use these imagined trajectories to update the policy and the critic.
        \ENDFOR
        \STATE $\verb|//|  Data\ Collection$
        \FOR{$m = 1,2,...,M$}
        \STATE $o_1\leftarrow \mathcal{M}_m.reset()$
        \FOR{$\text{sample step} = 1,2,..., S$}
        \STATE Compute $h_t,s_t$ and sample action $a_t$ via the policy.
        \STATE $r_t,o_{t+1}\leftarrow \mathcal{M}_m.step(a_t)$
        \ENDFOR
        \STATE Add these data to the replay buffer $\mathcal{D}_m$.
        \ENDFOR
        \ENDWHILE
    \STATE $\verb|//|  Model\ Evaluation$
    \FOR{$n = 1,2,...,N$}
    \STATE $o_1\leftarrow \mathcal{M}_{M+n}.reset()$
    \WHILE{the environment not done}
    \STATE Compute $h_t,s_t$ and sample action $a_t$ via the policy.
    \STATE $r_t,o_{t+1}\leftarrow \mathcal{M}_{M+n}.step(a_t)$
    \ENDWHILE
    \ENDFOR
    \end{algorithmic}  
\label{tad-algo}
\end{algorithm*} 

\section{Experimental details for DMControl}
\label{app_expe_intro}

\subsection{Details of all task combinations}

In this part, we first \emph{roughly} discuss the reward function of tasks in our experiments to better understand their TDR. These reward functions always defined by $\text{tolerance}$ function in DeepMind control suite~\cite{8}, which is a smooth function with parameters $\text{tolerance}(x, \text{bounds} = (\text{lower}, \text{upper}))$ and hope the value of $x$ is within $(\text{lower}, \text{upper})$. More details about $\text{tolerance}$ function can be found in \cite{8}.

\begin{itemize}
    \item \textbf{Cartpole-balance\&sparse.} This task combination includes two tasks: Cartpole-balance and Cartpole-balance$\_$sparse, which both hope to balance an unactuated pole with dense and sparse rewards, respectively. The optimal actions of these tasks are both hoped to balance the agent, and thus, TDR here is 0.
    \item \textbf{Walker-stand\&walk\&prostrate\&flip.} This task combination includes four tasks: Walker-stand, Walker-walk, Walker-prostrate, and Walker-flip.
    Walker-stand hopes the height of a two-legged robot to be larger than the target height. 
    Walker-walk hopes the height of the improved planar walker to be larger than a target height and the speed of the robot to be larger than another target speed.
    Walker-prostrate hopes the height of the robot to be lower than the target height.
    Finally, Walker-flip hopes the robot to stand and move forward to the target speed by executing a rapid twist and jump.
    In detail, their reward functions can be roughly described as
    \begin{equation}
    \begin{split}
        \mathcal{R}_{\text{stand}} =& \text{tolerance}(\text{height}, (1.2,\infty)),\\
        \mathcal{R}_{\text{walk}} =& \text{tolerance}(\text{height}, (1.2,\infty))
        * \text{tolerance}(\text{speed}, (1,\infty)). \\
        \mathcal{R}_{\text{prostrate}} =& \text{tolerance}(\text{height}, (0.0, 0.2)). \\
        \mathcal{R}_{\text{flip}} = & \text{tolerance}(\text{height}, (1.2,\infty))
        * \text{tolerance}(\text{angmomentum}, 5, \infty).
    \end{split}
    \end{equation}
    In this situation, for all states, the optimal actions of Walker-walk/stand and Walker-prostrate are opposite, and TDR here is huge.
    \item \textbf{Cheetah-run\&run$\_$back\&flip\&flip$\_$back.} This task combination includes four tasks: Cheetah-run, Cheetah-run$\_$back, Cheetah-flip, and Cheetah-flip$\_$back.
    Cheetah-run hopes to control a running planar biped to run forward within a target speed. 
    Cheetah-run$\_$back, differently, hopes to control the Cheetah robot to run backward within a target speed. 
    Cheetah-flip hopes the robot to move forward to the target speed by executing a rapid twist and jump.
    Similarly, Cheetah-flip$\_$back hopes to control the robot to move backward by flipping.
    \begin{equation}
    \begin{split}
        \mathcal{R}_{\text{Cheetah}\_\text{run}} = & \text{tolerance}(\text{speed}, (10, \infty)). \\
        \mathcal{R}_{\text{Cheetah}\_\text{run$\_$back}} = &\text{tolerance}(-\text{speed}, (10, \infty)). \\
        \mathcal{R}_{\text{Cheetah}\_\text{flip}} = & \text{tolerance}(\text{angmomentum}, (5, \infty)). \\
        \mathcal{R}_{\text{Cheetah}\_\text{flip$\_$back}} = & \text{tolerance}(-\text{angmomentum}, (5, \infty)).
    \end{split}
    \end{equation}
    Obviously, in this situation, for all states, the optimal actions of Cheetah-run and Cheetah-run$\_$back are opposite, and TDR here is also huge.
\end{itemize}

\begin{table*}[t]
\centering
\footnotesize
\tablestyle{3.5pt}{1.1}
\begin{tabular}{ccccccccc}
\toprule
\multirow{2}*{Tasks}& Acrobot-Cartpole-Pendulum & Walker-Cheetah-Hopper & \multicolumn{2}{c}{Cheetah-run$\_$mass} & \multicolumn{2}{c}{Walker-walk$\_$mass}\\
& Train\&Test & Train\&Test & Train & Test & Train & Test \\
\midrule
Dreamer
& 541.3 $\pm$ 4.0
& 299.4 $\pm$ 17.2
& 717.8 $\pm$ 27.9 
& 711.7 $\pm$ 38.8
& 889.9 $\pm$ 114.1 
& 903.3 $\pm$ 102.8\\
TAD
& \textbf{667.7 $\pm$ 6.4}
& \textbf{554.7 $\pm$ 23.6}
& \textbf{754.3 $\pm$ 22.9}
& \textbf{738.2 $\pm$ 41.1}
& \textbf{957.8 $\pm$ 35.1}
& \textbf{963.1 $\pm$ 32.6}\\
\bottomrule
\end{tabular}
\caption{Generalization performance (mean $\pm$ std) over different task distributions in image-based DMC of the best policy.
Numbers greater than \textbf{95 $\%$} of the best performance for each environment are \textbf{bold}.}
\label{table_dynamic}
\end{table*}

\subsection{Details of all task distributions}
Now we introduce the three task distributions in our experiments, which are designed based on existing tasks in the DeepMind control suite for testing the generalization of trained agents.
\begin{itemize}
    \item \textbf{Cheetah$\_$speed($\alpha,\beta$).} This task distribution is designed in this paper with parameter $0\leq \beta\leq \alpha$, based on the task Cheetah$\_$run in the DeepMind control suite, and hopes the Cheetah robot can run with the target speed interval $(\alpha-\beta, \alpha+\beta)$.
    \begin{equation}
    \begin{split}
        &\mathcal{R}_{\text{Cheetah}\_\text{speed}}(\alpha, \beta) = \text{tolerance}(\text{speed}, (\alpha-\beta, \alpha+\beta)).
    \end{split}
    \end{equation}
    We train the agents in tasks with parameters ($0.5,0.2$), ($1.5,0.2$), ($2.0,0.2$), ($3.0,0.2$) and test them in tasks with parameters ($1.0,0.2$), ($2,5,0.2$).
    \item \textbf{Pendulum$\_$angle($\alpha,\beta$).} This task distribution is designed in this paper with parameter $-1\leq \alpha\leq \beta\leq 1$, based on the task Pendulumh$\_$swingup in DeepMind control suite, and hopes the Pendulum robot can swing up within the target angle interval $(\arccos\alpha, \arccos\beta)$.
    \begin{equation}
    \begin{split}
        \mathcal{R}_{\text{Pendulum}\_\text{angle}}(\alpha, \beta) = \text{tolerance}(\text{angle}, (&\arccos\alpha, 
        \arccos\beta)).
    \end{split}
    \end{equation}
    Training tasks are with parameters ($-0.95,-0.9$), ($-0.85,-0.8$), ($-0.8,-0.75$), ($-0.7,-0.65$) and test tasks are with parameters ($-0.9,-0.85$), ($-0,75,-0.7$).
    \item \textbf{Walker$\_$speed($\alpha,\beta$).} This task distribution is designed in this paper with parameter $0\leq \beta\leq \alpha$, based on the task Walker$\_$run in DeepMind control suite, and hopes the Walker robot can run within the target speed interval $(\alpha-\beta, \alpha+\beta)$.
    \begin{equation}
    \begin{split}
        &\mathcal{R}_{\text{Walker}\_\text{speed}}(\alpha, \beta) = \text{tolerance}(\text{speed}, (\alpha-\beta, \alpha+\beta)).
    \end{split}
    \end{equation}
    We train the agents in tasks with parameters ($0.5,0.2$), ($1.5,0.2$), ($2.0,0.2$), ($3.0,0.2$) and test in tasks with parameters ($1.0,0.2$), ($2,5,0.2$).
\end{itemize}

Moreover, we introduce some details about our experiments. Our codes are based on Python and the deep learning library PyTorch. All algorithms are trained on one NVIDIA GeForce RTX 2080 Ti. Each seed and each task setting will take around 3 days. As for the hyper-parameters, we follow previous works~\cite{3,4,5} and select 2 as the action repeat for all experiments following~\cite{5}.

\subsection{Visualization results for task generalization}

Moreover, for each task sampled from the task distribution Cheetah$\_$speed (here parameters $(3.0, 0.2)$, $(2.0, 0.2)$, $(1.5, 0.2)$, $(0.5, 0.2)$ are for training tasks and parameters $(2.5, 0.2), (0.5, 0.2)$ are for test tasks), we plot the speed of the agent as a function of the timestep in Fig.~\ref{fig_cheetah}. As depicted, for each task, the agent trained by TAD will quickly improve its speed until it reaches the target speed and then keep its speed since the speed determines whether it has met the task requirements via utilizing historical information. Consequently, TAD is aware of different tasks and can successfully generalize to unseen test tasks. 
We also provide videos of these trajectories in the supplementary materials.

\begin{figure}[h]
\centering
\includegraphics[height=3.6cm,width=5.6cm]{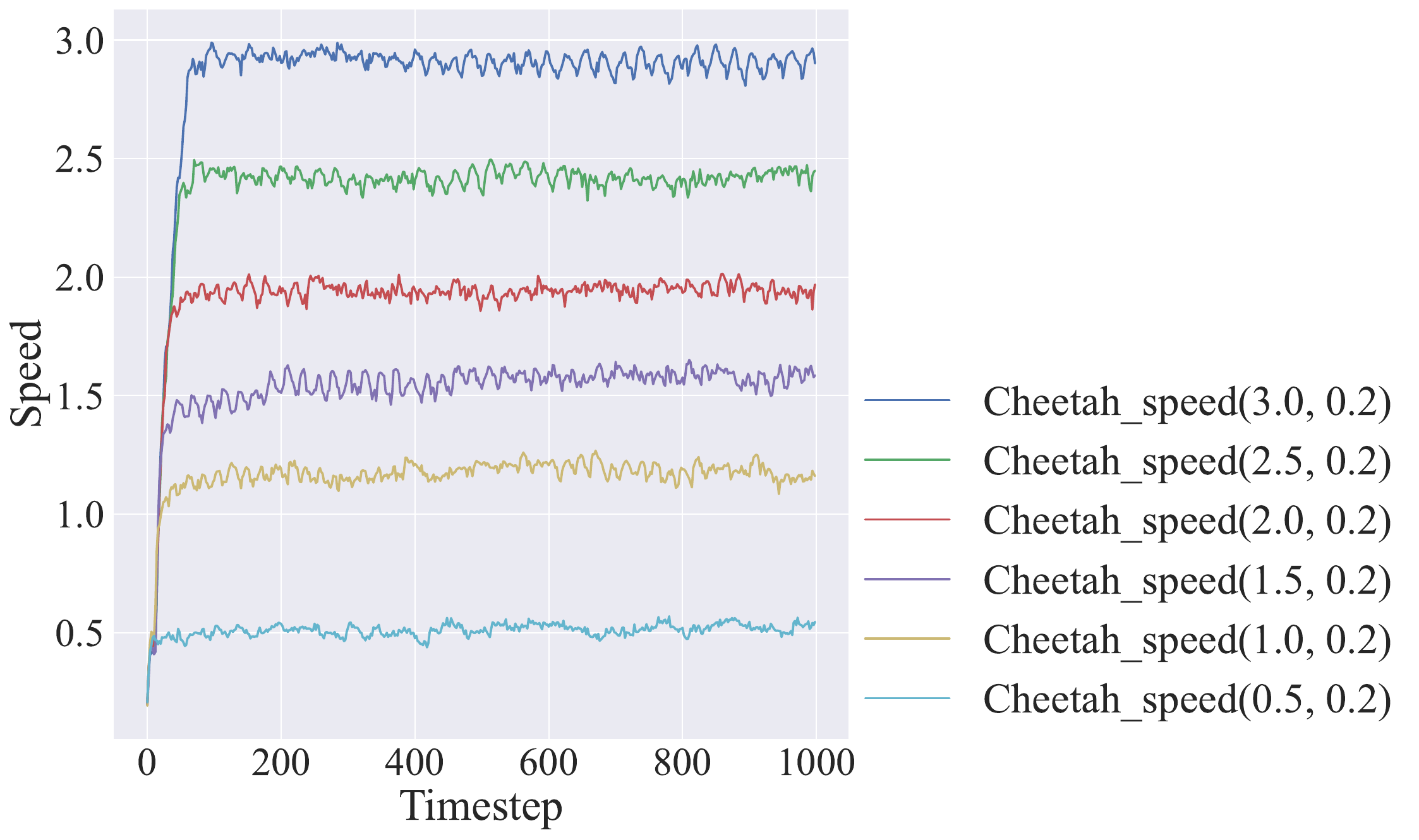}
\caption{
Visualization of the trained TAD agent in the task distribution of Cheetah$\_$speed. 
We plot the speed of the agent as a function of timesteps in all tasks. 
}
\label{fig_cheetah}
\end{figure}

\subsection{Implementation details}
TAD is implemented based on a widely adopted Dreamer repository {https://github.com/yusukeurakami/dreamer-pytorch} and we follow its default setups, including hyperparameters and network architecture. For the task model $p_\theta(\mathcal{M}|h_t,s_t)$, we choose to implement it with a simple 3-layer MLP, which maps current state $s_t$ and hidden state $s_t$ to a probability simplex $\triangle^{M-1}$ (assume we have $M$ tasks).

\subsection{TDR analyses in DMControl experiments}
We would like to illustrate the usage of TDR with a few examples in our experiments, including \textbf{Walker-stand$\&$walk} and \textbf{Walker-walk$\&$prostrate}.\\
To simplify the analyses, we set the reward of each state in Walker-stand $\mathcal{M}_1$ to 1 if the height of the robot is higher than 1.2 and 0 otherwise. In Walker-prostrate $\mathcal{M}_2$, we set the reward of each state to 1 if the height is lower than 0.2 and 0 otherwise. In Walker-walk, $\mathcal{M}_3$, we set the reward of each state to 1 if the height is higher than 1.2 and the speed is greater than 1.0; otherwise, the reward is 0.\\
We first estimate the optimal $Q^*$ function of these three tasks. For any state action pair $(s,a)$ in $\mathcal{M}_1$, if the height of the current state $s$ is higher than 1.2, the optimal policy is to maintain the height and we have $Q^*_{\mathcal{M}_1} = 1+\gamma+\gamma^2+...= \frac{1}{1-\gamma}$. Otherwise, if the height of $s$ is lower than 1.2, the optimal policy is to make the agent stand in one step and maintain a height of at least 1.2. We can deduce that $Q^*_{\mathcal{M}_1} = \gamma+\gamma^2+...= \frac{\gamma}{1-\gamma}$. The analyses of $Q^*_{\mathcal{M}_2}, Q^*_{\mathcal{M}_3}$ are similar and we conclude as follows:
\begin{equation}
\begin{split}
    Q^*_{\mathcal{M}_1}(s,a) & = 
    \begin{cases}
    \frac{1}{1-\gamma}, & \text{if height of $s\ge 1.2$} \\
    \frac{\gamma}{1-\gamma}, & \text{if height of $s< 1.2$}
    \end{cases} \\
    Q^*_{\mathcal{M}_2}(s,a) & = 
    \begin{cases}
    \frac{1}{1-\gamma}, & \text{if height of $s\leq 0.2$} \\
    \frac{\gamma}{1-\gamma}, & \text{if height of $s > 0.2$}
    \end{cases} \\
    Q^*_{\mathcal{M}_3}(s,a) & = 
    \begin{cases}
    \frac{1}{1-\gamma}, & \text{if height of $s\ge 1.2$ and speed of $s\ge 1.0$} \\
    \frac{\gamma}{1-\gamma}, & \text{if height of $s< 1.2$ or speed of $s< 1.0$}
    \end{cases} 
\end{split}
\end{equation}
Consider a task combination $\mathcal{T}_1$ that is distributed equally over Walker-stand and Walker-walk, i.e., $\mathbb{P}(\mathcal{M}_1) = \mathbb{P}(\mathcal{M}_3) = 0.5$. If height of $s \ge 1.2$ and speed of $s \ge 1.0$, $Q^*_{\mathcal{M}_1}(s,a) = Q^*_{\mathcal{M}_3}(s,a) = \frac{1}{1-\gamma}$. Thus, we have:
\begin{equation}
\begin{split}
    D_{\text{TDR}}(\mathcal{T}_1, s) 
    =& \mathbb{E}_{\mathcal{M}\sim\mathcal{T}} [\max_a Q_{\mathcal{M}}^*(s,a)]
    - \max_{a}\mathbb{E}_{\mathcal{M}\sim\mathcal{T}} [ Q_{\mathcal{M}}^*(s,a) ]
    = \frac{1}{1-\gamma} - \frac{1}{1-\gamma} = 0.
\end{split}
\end{equation}
Based on Theorem 3, for some policy $\pi\in \Pi_1, \Pi_2$:
\begin{equation}
    J_{\mathcal{T}}^* - \mathbb{E}_{\mathcal{M}\sim\mathcal{T}_1}\left[J_{\mathcal{M}}(\pi)\right]  
    \ge \frac{1}{1-\gamma}\mathbb{E}_{s\sim d_{\mathcal{M},\pi}} [D_{\text{TDR}}(\mathcal{T}_1, s)] = 0.
\end{equation}
Consequently, in the task distribution $\mathcal{T}_1$, there is no gap between the optimal performance achievable by the policies in $\Pi_1,\Pi_2$ and the actual optimal value. A TDR of $0$ indicates that policies in $\Pi_1$ and $\Pi_2$ are sufficient to achieve optimal performance in $\mathcal{T}_1$.. \\
In contrast, consider a task combination $\mathcal{T}_2$ that is distributed equally over Walker-walk and Walker-prostrate, i.e., $\mathbb{P}(\mathcal{M}_2) = \mathbb{P}(\mathcal{M}_3) = 0.5$. For any $s$, $\frac{1}{2} (Q^*_{\mathcal{M}_2}(s,a) + Q^*_{\mathcal{M}_3}(s,a)) \leq \frac{1}{2} (\frac{1}{1-\gamma} + \frac{\gamma}{1-\gamma})$. Thus, we have:
\begin{equation}
\begin{split}
    D_{\text{TDR}}(\mathcal{T}_2, s) 
    =& \mathbb{E}_{\mathcal{M}\sim\mathcal{T}} [\max_a Q_{\mathcal{M}}^*(s,a)]
    - \max_{a}\mathbb{E}_{\mathcal{M}\sim\mathcal{T}_2} [ Q_{\mathcal{M}}^*(s,a) ]
    \ge \frac{1}{1-\gamma} - \frac{1}{2}\frac{1+\gamma}{1-\gamma} = 0.5.
\end{split}
\end{equation}
Consequently, based on Theorem 3, for any policy $\pi\in \Pi_1, \Pi_2$:
\begin{equation}
    J_{\mathcal{T}}^* - \mathbb{E}_{\mathcal{M}\sim\mathcal{T}_2}\left[J_{\mathcal{M}}(\pi)\right]  
    \ge \frac{1}{1-\gamma}\mathbb{E}_{s\sim d_{\mathcal{M},\pi}} [D_{\text{TDR}}(\mathcal{T}_2, s)] = \frac{1}{2} \frac{1}{1-\gamma}.
\end{equation}
In other words, our conclusion provides a non-trivial gap between the optimal performance achievable by the policies in $\Pi_1,\Pi_2$ and the actual optimal value. A non-zero TDR shows that policies in $\Pi_1,\Pi_2$ can only obtain half the optimal performance in $\mathcal{T}_2$.\\
Our experimental results in Table~\ref{table_tad_multi_task} also confirm our analysis regarding the effect of TDR. While policies in $\Pi_1,\Pi_2$ perform appealingly in Walker-stand$\&$walk, their performance degenerates significantly in Walker-walk$\&$prostrate.

{
\begin{table*}[h]
\centering
\tablestyle{3.5pt}{1.1}
\footnotesize
\begin{tabular}{ccc}
\toprule
{Algorithm}& {Walker-stand$\&$walk} & {Walker-walk$\&$prostrate}
\\
\midrule
{CURL}
& {563.4 $\pm$ 16.0}
& {441.0 $\pm$ 7.7}
\\
{SAC+AE}
& {683.8 $\pm$ 164.3}
& {445.8 $\pm$ 2.6}
\\
{PlaNet}
& {\textbf{990.6 $\pm$ 3.8}}
& {445.4 $\pm$ 5.1}
\\
{Dreamer}
& {\textbf{992.3 $\pm$ 1.7}}
& {449.7 $\pm$ 12.5}
\\
{TAD}
& {\textbf{994.6 $\pm$ 0.1}}
& {\textbf{855.3 $\pm$ 56.2}}
\\
\bottomrule
\end{tabular}
\caption{{Performance (mean $\pm$ std) over different task combinations in DMC of the best policy.
Numbers greater than 95 \% of the best performance for each environment are bold.}}
\label{table_tad_multi_task}
\end{table*}
}

\subsection{MSE loss of imagination images in DMControl experiments}
In Table~\ref{table_mse_loss}, we have provided quantitative results on the models' prediction error of imagined observations. As shown below, the prediction error of Dreamer gradually increases over the time horizon, while TAD exhibits no significant increasing trend. We attribute this phenomenon to TAD’s ability to access diverse trajectories sampled from different tasks during training.

\begin{table*}[h]
\centering
\tablestyle{3.5pt}{1.1}
\footnotesize
\begin{tabular}{cccccccccccccccc}
\toprule
{Algorithm} & {Task} & {step 10} & {step 15} & {step 20} & {step 25} & {step 30} & {step 35}
\\
\midrule
\multirow{2}{*}{{Dreamer}}
& {Cheetah-run}
& {19.2}
& {18.8}
& {28.7}
& {29.3}
& {30.1}
& {32.5}
\\
& {Cheetah-flip}
& {18.6}
& {24.3}
& {34.8}
& {37.5}
& {33.5}
& {31.1}
\\
\multirow{2}{*}{{TAD (Ours)}}
& {Cheetah-run}
& {19.3}
& {23.2}
& {21.3}
& {23.9}
& {25.6}
& {23.3}
\\
& {Cheetah-flip}
& {17.9}
& {18.8}
& {21.1}
& {19.5}
& {17.3}
& {19.4}
\\
\bottomrule
\end{tabular}
\caption{{MSE Loss of imagination results of TAD and Dreamer}}
\label{table_mse_loss}
\end{table*}

\subsection{Additional results with task parameters}
The core idea of TAD is to regularize the world model to distinguish tasks with $p_\theta(\mathcal{M}|h_t,s_t)$, with $\mathcal{M}$ being the task indicator. While the current implementation uses a set of discrete training tasks, we would like to clarify that our method is not limited to discrete spaces of tasks. For continuous spaces of tasks, when the task parameters are available, we can optimize TAD by directly predicting the continuous \textbf{task parameters}. When the task parameters are unavailable, we can switch to the discrete implementation, where we randomly sample $M$ tasks to construct a set of discrete training tasks and optimize TAD by predicting the \textbf{task indices}. Obviously, the task indices are less informative than task parameters, because they can only identify different tasks, while providing no information about the relationships or similarities among them. We have provided additional experiments of TAD with access to continuous task parameters in Table~\ref{table_mujoco_continuous}, labeled as TAD-Continuous. Results in the table indicate that our method using the discrete implementation (TAD-CE/SC) achieves performance comparable to its continuous counterpart.

\begin{table*}[h]
\centering
\footnotesize
\resizebox{0.95\textwidth}{!}{
\begin{tabular}{cccccccccc}
\toprule
\multirow{2}*{{Algorithms}}& \multirow{2}*{{Hypothesis}} & {Half-Cheetah-Fwd-Back(1e7)} & \multicolumn{2}{c}{{Half-Cheetah-Vel(1e7)}} & \multicolumn{2}{c}{{Humanoid-Direc-2D(1e6)}}\\
& & {Train\&Test} & {Train} & {Test} & {Train} & {Test}\\
\midrule
{PlaNet}
& {$\Pi_2$}
& {30.5 $\pm$ 42.9}
& {-198.1 $\pm$ 1.9} 
& {-202.1 $\pm$ 1.8} 
& {215.9 $\pm$ 72.3}
& {220.6 $\pm$ 75.3} \\
{Dreamer}
& {$\Pi_2$}
& {127.4 $\pm$ 181.8} 
& {-151.4 $\pm$ 0.4}
& {-169.4 $\pm$ 1.2}
& {260.5 $\pm$ 48.9}
& {263.5 $\pm$ 52.3} \\
{RL2(zero-shot)}
& {$\Pi_3$}
& {1070.7 $\pm$ 109.7}
& {---}
& {-70.3 $\pm$ 6.7}
& {---}
& {191.9 $\pm$ 50.8} \\
{RL2(few-shot)}
& {$\Pi_3$}
& {1006.9 $\pm$ 26.4}
& {---}
& {-146.9 $\pm$ 0.4}
& {---}
& {268.8 $\pm$ 30.2} \\
{MAML(few-shot)}
& {---}
& {429.3 $\pm$ 81.4}
& {---}
& {-121.0 $\pm$ 37.1}
& {---}
& {205.3 $\pm$ 34.7} \\
{VariBAD(zero-shot)}
& {$\Pi_3$}
& {1177.5 $\pm$ 94.9}
& {---}
& {-58.4 $\pm$ 20.6}
& {---}
& {260.3 $\pm$ 61.6} \\
{TAD-CE (Ours)}
& {$\Pi_3$}
& {1455.8 $\pm$ 78.3}
& {\textbf{-49.3 $\pm$ 1.9}}
& {\textbf{-47.1 $\pm$ 0.3}}
& {\textbf{339.5 $\pm$ 78.7}}
& {\textbf{335.5 $\pm$ 70.5}} \\
{TAD-SC (Ours)}
& {$\Pi_3$}
& {\textbf{1541.5 $\pm$ 114.8}}
& {-\textbf{50.5 $\pm$ 1.6}}
& {-\textbf{49.6 $\pm$ 1.6}}
& {260.2 $\pm$ 185.0}
& {249.0 $\pm$ 168.9}\\
{TAD-Continuous (Ours)}
& {$\Pi_3$}
& {\textbf{1539.9 $\pm$ 102.0}}
& {\textbf{-47.8 $\pm$ 1.1}}
& {\textbf{-47.5 $\pm$ 1.3}}
& {\textbf{302.9 $\pm$ 36.3}}
& {\textbf{304.6 $\pm$ 39.0}} \\
\bottomrule
\end{tabular}
}
\caption{{Generalization performance (mean $\pm$ std) in MuJoCo.
Numbers greater than 95$\%$ of the best performance are \textbf{bold}.}}
\label{table_mujoco_continuous}
\end{table*}

\section{Experimental details for MuJoCo}
\label{appendix_mujoco}
We here introduce state-based control tasks, including Half-CheetahFwd-Back, Half-Cheetah-Vel, and Humanoid-Direc-2D, in detail, following the setting of previous meta RL works~\cite{27,29}.

\begin{itemize}
    \item \textbf{Half-Cheetah-Fwd-Back.} This task distribution includes two tasks: moving forward and moving backward.
    \item \textbf{Half-Cheetah-Vel.} This task distribution hopes the agent to move forward and achieve the target velocity. There are 100 training tasks and 30 test tasks for experiments.
    \item \textbf{Humanoid-Direc-2D.} This task distribution hopes the agent to move in the target direction. There are 100 training tasks and 30 test tasks for experiments.
\end{itemize}

Moreover, we introduce some details about our experiments. Our codes are based on Python and the deep learning library PyTorch. All algorithms are trained on one NVIDIA GeForce RTX 2080 Ti. Each seed and each task setting will take around 1 day. We select 1 as the action repeat for all following experiments.

\begin{table*}[t]
\centering
\footnotesize
\begin{tabular}{ccccccccc}
\toprule
\multirow{2}*{$\beta$}& \multicolumn{2}{c}{0.2} & \multicolumn{2}{c}{0.15} & \multicolumn{2}{c}{0.1}\\
& Train & Test & Train & Test & Train & Test\\
\midrule
Dreamer
& 250.2 $\pm$ 9.6
& 3.0 $\pm$ 2.2
& 247.5 $\pm$ 1.0
& 0.0 $\pm$ 0.0
& 175.8 $\pm$ 55.8
& 13.5 $\pm$ 13.5 \\
TAD
& \textbf{951.9 $\pm$ 3.3}
& \textbf{876.9 $\pm$ 51.1}
& \textbf{927.6 $\pm$ 2.6}
& \textbf{800.1 $\pm$ 121.4}
& \textbf{608.5 $\pm$ 321.4}
& \textbf{491.6 $\pm$ 389.4}\\
\bottomrule
\end{tabular}
\caption{Average cumulative reward  (mean $\pm$ one std) over different target region (smaller $\beta$ represents smaller target region and more sparse return) of the best policy trained by Dreamer and TAD in Cheetah$\_$Speed. For each $\beta$, we train agents in the train tasks and evaluate them in both train and test environments. Numbers greater than 95 percent of the best performance for each environment are \textbf{bold}.}
\label{result_sparse_1}
\end{table*}

\begin{table*}[t]
\centering
\footnotesize
\begin{tabular}{ccccccccc}
\toprule
\multirow{2}*{SR}& \multicolumn{2}{c}{0.0} & \multicolumn{2}{c}{0.8} & \multicolumn{2}{c}{0.9}\\
& Train & Test & Train & Test & Train & Test\\
\midrule
Dreamer
& 250.2 $\pm$ 9.6
& 3.0 $\pm$ 2.2
& 237.4 $\pm$ 14.9
& 28.9 $\pm$ 14.6
& 168.7 $\pm$ 96.2
& 157.2 $\pm$ 188.4 \\
TAD
& \textbf{951.9 $\pm$ 3.3}
& \textbf{876.9 $\pm$ 51.1}
& \textbf{841.9 $\pm$ 154.9}
& \textbf{546.1 $\pm$ 291.6}
& \textbf{777.7 $\pm$ 170.1}
& \textbf{716.8 $\pm$ 136.5}\\
\bottomrule
\end{tabular}
\caption{Average cumulative reward (mean $\pm$ one std) over different sparse rates of the best policy trained by Dreamer and TAD in Cheetah$\_$Speed. For each sparse rate, we train agents in the train tasks and evaluate them in both train and test environments. Numbers greater than 95 percent of the best performance for each environment are \textbf{bold}.}
\label{result_sparse}
\end{table*}

\section{Ablation study}
\subsection{Dynamic generalization}

As TAD utilizes all historical information to infer the environment, it can be directly applied to more general settings with different observations, dynamics, and/or actions. To evaluate the performance of TAD in these settings, we have conducted the following four experiments based on DMControl:

\begin{itemize}
    \item \textbf{Acrobot-Cartpole-Pendulum}: includes 7 tasks of artpole-balance, cartpole-balance$\_$sparse, cartpole-swingup, cartpole-swingup$\_$sparse, acrobot-swingup, acrobot-swingup$\_$sparse, and pendulum-swingup. All these tasks aim to control a rod-shaped robot, while they own different \textbf{embodiments}, \textbf{dynamics}, and \textbf{observations}.
    \item \textbf{Walker-Cheetah-Hopper}: includes 6 tasks of walker-prostrate, walker-stand, walker-walk, cheetah-run, hopper-stand, and hopper-hop. The tasks own different \textbf{embodiments}, \textbf{dynamics}, \textbf{actions}, and \textbf{observations}.
    \item \textbf{Cheetah-run-mass($m$)}: This task distribution is based on the task Cheetah-run, and the mass of the robot is $m$ times that of the standard task. Thus different tasks own different \textbf{dynamics}. We train the agents in tasks with $m=0.6,1.0,1.2,1.6$ and test them in tasks with parameters $m=0.8, 1.4$.
    \item \textbf{Walker-walk-mass($m$)}: This task distribution is based on the task Walker-walk, and the mass of the robot is $m$ times that of the standard task. Thus different tasks own different \textbf{dynamics}. We train the agents in tasks with $m=0.6,1.0,1.2,1.6$ and test them in tasks with parameters $m=0.8, 1.4$.
\end{itemize}
Then, we test Dreamer and TAD in these four settings and report the results. TAD can achieve much greater performance and better convergence compared to Dreamer, as it can better infer the current task. This experiment indicates TAD’s potential in further handling dynamic generalization and even cross-embodiment tasks.

\subsection{Sparse reward}

In this part, we will evaluate TAD in more challenging settings with sparse rewards. First, we evaluate Dreamer and TAD in Cheetah$\_$speed with different $\beta$, which identifies the region of target speeds. With smaller $\beta$, the reward signals are more sparse since the target intervals are smaller. In the main experiment, we take $\beta=0.2$, and here we evaluate in $\beta=0.2,0.15,0.1$, of which the result is reported in Table~\ref{result_sparse_1}. For each $\beta$, we take the training parameters ($0.5,\beta$), ($1.5,\beta$), ($2.0,\beta$), ($3.0,\beta$) and test them in tasks with parameters ($1.0,\beta$), ($2.5,\beta$). As shown in Table~\ref{result_sparse_1}, with the decreasing of $\beta$, the performance of Dreamer and TAD decreases since reward signals are sparse, so exploration here is much more difficult. However, our TAD still significantly outperforms Dreamer and shows strong generalization abilities, which shows that TAD can effectively utilize historical information, even sparse rewards.  

Moreover, we design Cheetah$\_$speed$\_$sparse based on Cheetah$\_$speed. In Cheetah$\_$speed$\_$sparse($n$), we make the reward function sparse, i.e., the output reward is the same as Cheetah$\_$speed every $n$ timesteps (in step $n-1, 2n-1, ...$) and 0 otherwise, of which the sparse rate (SR) is $(n-1)/n$. We supplement experiments to evaluate the performance of Dreamer and TAD with $n=5$ (SR=0.8) and $n=10$ (SR=0.9). As shown in Table~\ref{result_sparse}, with the increasing of SR, although the performance of TAD decreases since inferring task context from sparse reward is extremely difficult, TAD still shows strong performance in train tasks and generalizes well to unseen test tasks.

\end{appendix}



\end{document}